\RequirePackage{fix-cm}
\documentclass{article}

\usepackage{amsmath,amssymb,amsfonts}
\usepackage[numbers]{natbib}

\usepackage[justification=centering]{caption}

\usepackage{graphicx}

\usepackage[caption=false]{subfig}

\usepackage{multirow}
\usepackage[table]{xcolor}

\usepackage{amsthm}

\usepackage[super]{nth}

\newtheorem{thm}{Theorem}
\newtheorem{prop}{Proposition}
\newtheorem{example}{Example}
\newtheorem{lemma}{Lemma}

\usepackage{algorithm}
\usepackage[noend]{algpseudocode}

\algrenewcommand\textproc{}

\algdef{SE}[DOWHILE]{Do}{doWhile}{\algorithmicdo}[1]{\algorithmicwhile\ #1}%

\algnewcommand\algorithmicforeach{\textbf{for each}}
\algdef{S}[FOR]{Foreach}[1]{\algorithmicforeach\ #1\ \algorithmicdo}

\usepackage{xcolor}
\definecolor{coolblack}{rgb}{0.0, 0.18, 0.39}
\definecolor{darkred}{rgb}{0.75, 0.0, 0.0}

\newcommand{\hide}[1]{}

\usepackage{authblk}

\begin{document}

\title{Efficient Multiple Constraint Acquisition\footnote{This paper is an extended  
version of paper~\cite{mquacq} that appeared in the proceedings of CP-2018.} }


\author{Dimosthenis C. Tsouros}
\author{Kostas Stergiou}


\affil{
Dept. of Electrical \& Computer Engineering, \\ University of Western Macedonia,\\ Kozani, Greece\\
              \texttt{dtsouros@uowm.gr, kstergiou@uowm.gr} }

\date{}

\maketitle

\begin{abstract}
Constraint acquisition systems such as QuAcq and MultiAcq can assist non-expert users to model their problems as constraint networks by classifying (partial) examples as positive or negative. For each negative example, the former focuses on one constraint of the target network, while the latter can learn a maximum number of constraints. Two bottlenecks of the acquisition process where both these algorithms encounter problems are the large number of queries required to reach convergence, and the high cpu times needed to generate queries, especially near convergence. In this paper we propose algorithmic and heuristic methods to deal with both these issues. We first describe an algorithm, called MQuAcq, that blends the main idea of MultiAcq into QuAcq resulting in a method that learns as many constraints as MultiAcq does after a negative example, but with a lower complexity. A detailed theoretical analysis of the proposed algorithm is also presented.
Then we turn our attention to query generation which is a significant but rather overlooked part of the acquisition process. We describe 
how query generation in a typical constraint acquisition system operates, and we propose heuristics for improving its efficiency. Experiments from various domains demonstrate that our resulting algorithm that integrates all the new techniques does not only generate considerably fewer queries than QuAcq and MultiAcq, but it is also by far faster than both of them, in average query generation time as well as in total run time, and also largely alleviates the premature convergence problem. 
\end{abstract}

\section{Introduction}
\label{intro}

Constraint programming (CP) has made significant progress over the last decades, and is now considered as one of the foremost paradigms for solving combinatorial problems. The basic assumption in CP is that the user models the problem and a solver is then used to solve it. Despite the many successful applications of CP on combinatorial problems from various domains, there are still challenges to be faced in order to make CP technology even more widely used. 

A major bottleneck in the use of CP is modeling. Expressing a combinatorial problem as a constraint network requires considerable expertise in the field \cite{freuder1999modeling}. To overcome this obstacle, several techniques have been proposed for modeling a constraint problem automatically, and nowadays automated modeling is regarded as one of the most important aspects of CP \cite{freuder1999modeling,o2010automated,freuder2014grand,lombardi2018boosting,de2018learning,freuder2018progress}. Along these lines, an area of research that has started to attract a lot of attention is that of {\em constraint acquisition} where the model of a constraint problem is acquired (i.e. learned) using a set of examples that are posted to a human user or to a software system \cite{bessiere2017constraint,de2018learning}.  

Constraint acquisition is an area where CP  
meets machine learning, as it can be formulated as a 
concept learning task.
Constraint acquisition can come in various flavours depending on factors such as whether the learner can post queries to the user dynamically, and the type of queries that can be posted and answered. In {\em passive} acquisition, examples of solutions and non-solutions are provided by the user. Based on these examples and their classification by the user as positive (solutions) or negative (non-solutions), 
the system learns a set of constraints that correctly classifies all the given examples.
A limitation of passive acquisition (and passive learning in general) is the requirement, from the user's part, to provide diverse examples of solutions and non-solution to the system, especially in problems with irregular structure. 

In contrast, in {\em active} or {\em interactive} acquisition, the learner interacts with the user dynamically while acquiring the constraint network. 
In such systems, the basic query is to ask the user to classify an example as solution or not solution. 
This "yes/no" type of question is called membership query~\cite{angluin1988queries}, 
and this is the type of query that has received the most attention in active constraint acquisition.

A state-of-the-art interactive acquisition algorithm is QuAcq \cite{bessiere2013constraint}. For each example that is classified as negative by the user, QuAcq is able to learn one constraint of the target network by setting a series of partial queries to the user. An alternative algorithm, called MultiAcq, finds all the constraints of the target network violated by a generated example that is classified as negative \cite{arcangioli2016multiple}. However, MultiAcq needs a linear number of queries in the size of the example to learn each constraint, whereas QuAcq has a logarithmic complexity. 

In general, active acquisition has several advantages. First of all, it decreases the number of examples necessary to converge to the target set of constraints. In addition, it does not require the existence of diverse examples of solutions or non-solutions to the problem. This is a significant advantage especially if the problem has not already been solved. 
Another advantage is that the user does not need to be human. It might be a previous system developed to solve the problem. For example, active learning has been used to automatically acquire CSPs which model the elementary actions of a robot by asking queries to the simulator of the robot \cite{paulin2008automatic}.

However, active learning still presents computational challenges regarding the number of queries required and the cpu time needed to generate them. 
Despite the good theoretical bound of QuAcq and QuAcq-like approaches in terms of the number of queries required to learn a network, the generation of a membership query is an NP-complete problem. Hence, it can be too time-consuming, and therefore annoying, especially if the system interacts with a human user. For example, QuAcq can take more than 30 minutes to generate a query for the model acquisition of Sudoku puzzles near convergence. 

In this work, we present methods to deal with the challenges of interactive learning. We first introduce an algorithm, called MQuAcq, that blends the main idea of MultiAcq into QuAcq, achieving a better complexity bound than MultiAcq. This algorithm uses the reasoning of QuAcq when searching for constraints to learn once a negative query is encountered, but instead of focusing on one constraint, it learns a maximum number of constraints, just like MultiAcq does. But whereas MultiAcq learns constraints of the target network in a number of queries linear in the size of the example, our proposed approach finds constraints in a logarithmic number of queries. 
We make a detailed theoretical analysis of MQuAcq, proving its correctness and its complexity in terms of required queries.
We also propose an optimization on the process of locating scopes that reduces the number of queries needed to learn the target constraint network significantly. This is done by avoiding posting redundant queries to the user. 

Then we focus on the query generation process, which is an important step of constraint acquisition that has not been discussed in detail in the literature. We describe the query generation process of standard interactive acquisition systems in detail, and we propose heuristics that can be applied during query generation to boost the performance of constraint acquisition algorithms.

First, we present a heuristic that generalizes the idea of allowing partial queries to be posted to the user. Instead of using partial queries only when trying to focus on one or more constraints after a complete example has been classified as negative, we allow the generation of partial examples to be posted as partial queries to the user. As experiments demonstrate, this can reduce the time needed for the system to converge, resulting in avoidance of premature convergence and reduced total run time for the acquisition process. Next, we focus on the generation of more informative queries, i.e. queries that can help to reduce the version space faster.
We propose a variable ordering heuristic for the query generation process, aiming at generating queries with more information, and achieving to reduce the maximum cpu time needed for query generation. We then propose a value ordering heuristic that cuts down the number of queries required.

Experimental results with benchmark problems from various domains demonstrate that the integration of our methods results in an algorithm that considerably outperforms both QuAcq and MultiAcq as it generates significantly fewer queries, it is up to one order of magnitude faster in average query generation time, it is by far superior in total run time, and it largely alleviates the premature convergence problem from which both QuAcq and MultiAcq suffer.

In addition, experiments show that our proposed algorithm scales up quite well in terms of the number of queries required, while the time performance, being highly dependant on the size of the target network, can rise sharply. Also, it is shown that learning problems with expressive biases scales well with our proposed algorithm, even when using a big language to construct the bias, especially regarding the number of generated queries.

The rest of this paper is organized as follows. Related work is presented in Section~\ref{sec:rel}. Section~\ref{sec:background} introduces the necessary background on constraint acquisition. In Section~\ref{sec:quacq} we review the algorithms QuAcq and MultiAcq. Section~\ref{sec:mquacq} describes the new algorithm that we propose. In Section~\ref{findScope2} we describe the optimization to the process of locating scopes.
Section~\ref{sec:query} details the query generation process. In Section~\ref{sec:heur} we describe the proposed heuristics.
An experimental evaluation is presented at Section~\ref{sec:experiments}. In Section~\ref{sec:disc} we discuss some important aspects of MQuAcq and point to future work, while Section~\ref{sec:conclusion} concludes the paper.

\section{Related Work}
\label{sec:rel}

An early approach to passive constraint acquisition is the algorithm ConAcq.1 \cite{bessiere2004leveraging,bessiere2005sat,bessiere2017constraint}. Based on the examples of solutions and non-solutions provided by the user, the system learns a set of constraints that correctly classifies all examples given so far. A passive learning method based on inductive logic programming was proposed in \cite{lallouet2010learning}. This system uses background knowledge on the structure of the problem to learn a representation of the problem, correctly classifying the examples given. Another approach to passive learning is the ModelSeeker system \cite{beldiceanu2012model}. In this approach, the user has to provide positive examples to the system, which are then arranged as a matrix. Then the system uses the global constraints catalog to identify specific global constraints that are present in the model. ModelSeeker has been shown to be very effective in extracting a model from highly structured problems, requiring only a few positive examples to learn the model of problems such as Sudoku.

Concerning active acquisition, an early related work is the {\em matchmaker} agent which interactively asks the user to explain why a proposed example is considered as incorrect, by providing one of the constraints that are violated \cite{freuder1998suggestion}. 
An approach to interactive constraint acquisition using version spaces is presented in \cite{o2004study}. With this approach, examples are provided by the user that can be used to identify a version space of constraints and then the system attempts to generalize the user's examples by choosing a hypothesis from the current version space.
Another active learner presented in the literature is ConAcq.2 \cite{bessiere2007query,bessiere2017constraint} and its extension \cite{shchekotykhin2009argumentation}. Both these systems acquire constraint models using membership queries that are posted to the user \cite{angluin1988queries}. 

A state-of-the-art interactive acquisition algorithm, again based on membership queries, is QuAcq \cite{bessiere2013constraint}. QuAcq is able to ask the user to classify partial queries, i.e. incomplete variable assignments, which may be easier for the user to answer. Also, asking partial queries gives the system the capability to focus on the scope of a constraint that is violated and hence learn the constraint. If the answer to a membership query posted by QuAcq is positive, the system reduces the search space by removing the set of constraints violated by this example. If the answer is negative, QuAcq asks a series of partial queries to locate the scope of one of the violated constraints of the target network. 

An attempt to make QuAcq more efficient was presented by Arcangioli et. al with the MultiAcq algorithm \cite{arcangioli2016multiple}. Given that the generation of a useful membership query is not easy and that there may be several constraints that can be learned by each query, it is very likely that the system can learn more information from a negative query. So, instead of learning only one constraint, MultiAcq finds all the constraints of the target network violated by a generated example that is classified as negative. However, MultiAcq needs a linear number of queries in the size of the example to learn a constraint. On the other hand, QuAcq has a logarithmic complexity in terms of the number of queries. 

Attempts to reduce the time needed to generate a membership query were presented in \cite{addi2018time,pquacq}. Apart from membership queries, other types of queries such as recommendation and generalization ones, have also been proposed to be used in interactive constraint acquisition \cite{daoudi2016constraint,bessiere2014boosting}. The use of such queries can reduce the total number of queries needed to learn a model, but the drawback is that they require a higher level of expertise from the user's part.

Active constraint acquisition is a special case of query-directed learning, also known as {\em exact learning} \cite{BSHOUTY1995146,bshouty2018exact}. Concept learning via queries has been widely studied in the theoretical machine learning literature. There are well-known results for several classes of concepts~\cite{bshouty1996asking,bshouty2018exact}, e.g. conjunctions of Horn clauses~\cite{angluin1992learning}, $k$-term DNF~\cite{blum1995fast} (or CNF) formula, decision trees~\cite{BSHOUTY1995146} etc.

In the learning model introduced in~\cite{angluin1988queries} and used by the most approaches to exact learning, two types of queries are used. The {\em membership} query, mentioned above, that requests the user to classify a given example as positive or negative and the {\em equivalence} query, which asks the user to decide whether the given concept is equivalent to the target concept. In case of a negative answer, the user must then provide a counterexample that proves why the two concepts are not equivalent. 

As stated in~\cite{bessiere2017constraint}, in the context of constraint acquisition, posting equivalence queries to the user and expecting counterexamples to be returned is not feasible from a practical point of view as the assumption is that the user does not know the constraint network and does not have the skills to model the target concept directly. However, in the theoretical case where the user can answer equivalence queries, there exist theoretical results proving that a constraint network is learnable by equivalence queries alone~\cite{bessiere2017constraint}.
It has to be noted that constraint networks are quite complex to acquire and the results of generic concept learning algorithms cannot directly be compared to constraint acquisition algorithms. Also, operating with a bias of bounded arity constraints and without handling disjunctions of constraints, the current constraint acquisition algorithms cannot be applied to learn most of the boolean functions which have been studied in exact learning.

\section{Background}
\label{sec:background}

\subsection{Vocabulary and Constraint Networks}
\label{sec:voc}

The \textit{vocabulary} $(X, D)$ is a finite set of $n$ variables $X = \{x_1 , . . . , x_n \}$ and a set of domains $D = \{D(x_1), . . . , D(x_n)\}$, where $D(x_i) \subset \mathbb{Z}$ is the finite set of values for $x_i$. The vocabulary is the common knowledge shared by the user and the constraint acquisition system. 

A \textit{constraint} $c$ is a pair (rel($c$), scope($c$)), where scope($c$) $\subseteq X$ is the \textit{scope} of the constraint and rel($c$) is a relation between the variables in scope($c$) that specifies which of their assignments are allowed. $|scope(c)|$ is called the \textit{arity} of the constraint. We denote by $c_{ij}$ a binary constraint between variables $x_i$ and $x_j$, with $c$ being the relation. A \textit{constraint network} is a set $C$ of constraints on the vocabulary $(X, D)$. 
A constraint network that contains at most one constraint for each subset of variables (i.e. for each scope) is called a {\em normalized constraint network}.

An example $e_Y$ is an assignment on a set of variables $Y \subseteq X$. $e_Y$ is rejected by a constraint $c$ iff scope($c$) $\subseteq Y$ and the projection $e_{scope(c)}$ of $e_Y$ on the variables in scope($c$) 
is not in rel($c$). An assignment $e_Y$ is a partial solution iff it is accepted by all the constraints $c \in C$ where $scope(c) \subseteq Y$. A complete assignment that is accepted by all the constraints in $C$ is a solution to the problem. $sol(C)$ denotes the set of solutions of $C$. A partial assignment $e_Y$ which is accepted by $C$ is not necessarily part of a complete solution.

A \textit{redundant} or \textit{implied} constraint $c \in C$ is a constraint that if removed from the constraint network, the set of solutions $sol(C)$ remains the same. In other words, if all the other constraints in $C$ are satisfied then $c$ is also satisfied.

\subsection{Interactive Learning}
\label{sec:inter}

Using terminology from machine learning, a \textit{concept} is a Boolean function over $D^X = \prod_{x_i \in X} D(x_i) $, that assigns to each example $e \in D^X$ a value in $\{0, 1\}$, or in other words classifies it as negative or positive. The target concept $f_T$ is a concept that assigns 1 to $e$ if $e$ is a solution of the problem and 0 otherwise. In constraint acquisition, the target concept is the target constraint network $C_T$, such that $sol(C_T) = \{e \in D^X \mid f_T(e) = 1\}$.  Hereafter, following the literature, we will assume that the target constraint network is normalized.

Besides the vocabulary, the learner has a \textit{language} $\Gamma$ consisting of {\em bounded arity} constraints.
The \textit{constraint bias} $B$ is a set of constraints on the vocabulary $(X, D)$, built using the constraint language $\Gamma$, from which the system can learn the target constraint network. $\kappa_B(e_Y)$ represents the set of constraints in $B$ that reject $e_Y$.

The classification question asking the user to determine if an example $e_X$ is a solution to the problem that the user has in mind is called a \textit{membership query} $ASK(e)$. The answer to a membership query is positive if $f_T(e) = 1$ and negative otherwise. A \textit{partial query} $ASK(e_Y)$, with $Y \subseteq X$, asks the user to determine if $e_Y$, which is an assignment in $D^Y$, is a partial solution or not, i.e. if it is accepted by all the constraints $c \in C$ where $scope(c) \subseteq Y$. A classified assignment $e_Y$ is labelled as positive or negative depending on the answer of the user to $ASK(e_Y)$. Following the literature, we assume that all queries are answered correctly by the user. From now on, we will sometimes use the terms query and example interchangeably.

A {\em minimal scope} in a negative example $e_Y$ is a subset of variables $S \subseteq Y$ such that ASK($e_S$) = ``no'' and for all $x_i \in S$, ASK($e_{S \setminus x_i}$) = ``yes''. Thus, a minimal scope $S$ is the scope of a violated constraint $c \in C_T$, such that there does not exist any violated constraint $c^{\prime} \in C_T$ with $scope(c) \subset S$.

To better understand the terms presented, consider the following example, which we use as a running example in the rest of the paper.

\begin{example}
\label{ex:back}

Consider a problem consisting of $8$ variables with domains $\{ 1, ... , 8 \}$. The vocabulary $(X,D)$ given to the system would be $X = \{ x_1, ..., x_8 \}$ and $D = \{D(x_1), . . . , D(x_8)\}$ with $D(x_i) = \{ 1, ... , 8 \}$. Assume that the problem the user has in mind has to satisfy the constraints $x_1 \neq x_2$, $x_1 \neq x_3$ and $x_3 \neq x_4$. Thus, the target network $C_T$ would be the set $\{ \neq_{12}, \neq_{13}, \neq_{34} \}$. Also, for simplicity assume that the language $(\Gamma)$ given to the system by the user contains only the relation $\{\neq$\}. In this case, the bias $B$ would contain the given relation for all the possible scopes. As it is a binary relation, $B = \{ \neq_{ij}$ $\mid$ $1 \leq i < 8 \land i < j \leq 8\}$. In addition, given an example $e = \{1,1,1,2,3,4,5,6\}$, $\kappa_B(e) = \{\neq_{12}, \neq_{13}, \neq_{23}\}$. The scopes $\{x_1,x_2\}$, $\{x_1,x_3\}$, $\{x_2,x_3\}$ are minimal scopes, as there is no constraint in $C_T$ with a scope $S$ being a subset of the scope of any of them. If $e$ is posted to the user to be classified as positive or negative then ASK($e$) is a complete membership query. If only a partial assignment, for instance $e_Y= \{1,1,1,2,-,-,-,-\}$ with $Y = \{x_1,x_2,x_3,x_4\}$, is posted to the user then $e_Y$ is a partial query. 

\end{example}

In interactive constraint acquisition the system generates a set $E$ of complete and partial examples, which are labelled by the user as positive or negative. A constraint network $C$ agrees with $E$ if $C$ accepts all examples labelled as positive in $E$ and rejects those labelled as negative. The \textit{learned network} $C_L$ has to agree with $E$.

A (complete or partial) query $q = e_Y$ is called \textit{irredundant} (or \textit{informative}) iff the answer to $q$ is not predictable.  That is, $q$ is irredundant iff it is not classified as positive by all the constraints in the bias $B$, which means that $\kappa_B(e_Y)$ is not empty. At the same time, $q$ should be accepted by the learned network $C_L$ otherwise it will be classified as negative. 
Considering Example~\ref{ex:back}, the query $q$ = ASK($e$) is irredundant as it is not classified as positive by all the constraints from $B$. In this case, a query $q$ = ASK($e$), with $e = \{1,2,3,4,5,6,7,8\}$, would be redundant as it does not violate any constraint from $B$, thus it is surely positive.

The acquisition process has \textit{converged} on the learned network $C_L \subseteq B$ iff $C_L$ agrees with $E$ and for every other network $C \subseteq B$ that agrees with $E$, we have $sol(C) = sol(C_L)$. If the first condition is true ($C_L$ agrees with $E$) but the second condition has not been proved, we have \textit{premature convergence}.
If there does not exist a constraint network $C \subseteq B$ such that $C$ agrees with $E$ then the acquisition process has \textit{collapsed}. This happens when the target constraint network is not included in the bias, i.e. $C_T \nsubseteq B$.

\section{Algorithms for Interactive Constraint Acquisition}
\label{sec:quacq}

In this section we describe the state-of-the-art QuAcq and MultiAcq algorithms for interactive constraint acquisition. 

State-of-the-art constraint acquisition algorithms are based on the version space learning paradigm. Initially, the system uses the given language $\Gamma$ to construct the bias $B$, containing the ``candidate'' constraints. Then the system iteratively posts a series of membership queries to the user in order to learn the constraints of the target network. Each example posted as a query must satisfy $C_L$, i.e. the network that has already been learned so far,  
and violate at least one constraint from $B$. A query that satisfies the above criteria is called {\em informative}, as whatever the user's answer is, the version space will be pruned. In case the answer is positive, each constraint $c \in B$ that violates the posted example can be removed from $B$ (i.e. all the constraint networks containing $c$ are removed from the version space). In case the answer is negative, we know that one or more of the violated constraints are certainly in $C_T$. So, the system will search to find the scope of one or all of them, depending on the 
algorithm used. 
This is done through a function called {\em FindScope} in QuAcq. A similar function, called {\em FindAllScopes}, is used by MultiAcq.
Once a scope has been located, the function {\em FindC}~\cite{bessiere2013constraint} is used to learn the specific constraint (i.e. its relation).

\subsection{QuAcq}

QuAcq learns one constraint via each generated negative query. Once a generated example is classified as negative, QuAcq calls the recursive function {\em FindScope}  to discover the scope of one of the violated constraints, as follows. It successively maps the problem of finding a constraint to a simpler problem by removing entire blocks of variables from the query and asking partial queries to the user. If after the removal of some variables the answer of the user to the partial query posted is ``yes'', then we know that the removed block contains at least one variable from the scope of a violated constraint. So, then {\em FindScope} focuses on this block. When, after repeatedly removing variables, the size of the considered block becomes 1 (i.e. the block contains a single variable), this variable certainly belongs to a violated constraint. {\em FindScope} achieves a logarithmic complexity in terms of the number of queries posted to the user by splitting the variables approximately in half after each removal. 

QuAcq (Algorithm \ref{alg:quacq}) starts with an empty $C_L$ and a bias $B$ containing all the possible constraints that can be built using the constraint language $\Gamma$. The algorithm iteratively posts queries to the user, in the form of complete assignments. According to the classification of each query, the learned network $C_L$ is augmented with a new constraint from $B$ or some constraints are removed from $B$. If the algorithm terminates having learned the target network, it has converged. Otherwise, it has collapsed. 

In more detail, QuAcq first checks if the currently learned network has at least one solution. This is done in case the problem that the user has in mind is unsolvable. If indeed it is, the acquisition process collapses (line 3). Otherwise, QuAcq generates a complete assignment $e$ which satisfies the currently built $C_L$ and is rejected by at least one constraint in $B$ (line 4). 
This is an important step that is not described in detail in the literature. We focus on the query generation step in Section~\ref{sec:query}. 
If no such example exists, then the system has converged to the target network. After generating a suitable example $e$, this example is posted as a membership query to the user. If $e$ is classified as positive (i.e. it is a solution) then all constraints that violate it are removed from $B$ (line 6), as they definitely cannot be part of the target network. If $e$ is negative, the algorithm tries to find one constraint that is violated by $e$ to add to $C_L$ by calling functions {\em FindScope} and {\em FindC}  (line 8). Once the system learns the constraint (line 10), if no collapse occurs (line 9), it returns to the query generation step.  


\begin{algorithm}
\caption{QuAcq: Quick Acquisition}\label{alg:quacq}
\begin{footnotesize}
\begin{algorithmic}[1]

\Require $B$, $X$, $D$ ($B$: the bias, $X$: the set of variables, $D$: the set of domains)
\Ensure $C_L$ : a constraint network

\State $C_L \leftarrow \emptyset$;

\While {true}

	\If{ $sol(C_L) = \emptyset$ } \Return ``collapse'';
	\EndIf

	\State Generate $e$ in $D^X$ accepted by $C_L$ and rejected by $B$;

	\If{ $e$ = nil } \Return ``$C_L$ converged'';
	\EndIf

	\If{ $ASK(e)$ = ``yes'' }  $B \leftarrow B  \setminus  \kappa_B(e) $;
	\Else

		\State $c \leftarrow FindC(e, FindScope( e, \emptyset, X, false ) )$;

		\If{ $c$ = nil } \Return ``collapse'';
		\Else \quad $C_L \leftarrow C_L \cup \{c\} $;
		\EndIf

	\EndIf

\EndWhile

\end{algorithmic}
\end{footnotesize}
\end{algorithm}


{\em FindScope} (Algorithm \ref{alg:findscope}) takes as parameters an example $e$ that violates at least one constraint from the bias, two sets of variables $R$ and $Y$,
and a Boolean variable $ask\_query$. In the first call to {\em FindScope} from QuAcq, $e$ is the example generated in line 4 of QuAcq that is classified as negative, $R$ is the empty set and $Y$ the set of all the variables of the problem. $ask\_query$ is set to false as we already know that $e$ is a negative example.


\begin{algorithm}
\caption{{\em FindScope}}\label{alg:findscope}
\begin{algorithmic}[1]

\Require $e$, $R$, $Y$, $ask\_query$ ($e$: the example, $R$,$Y$: sets of variables, $ask\_query$: boolean)
\Ensure $Scope$ : a set of variables, the scope of a constraint in $C_T$

\Function{FindScope}{$e$, $R$, $Y$, $ask\_query$}	

	\If{ $ask\_query$ } 

		\If{ $ASK(e_R)$ = ``yes'' } $B \leftarrow B  \setminus  \kappa_B(e_R) $;
		\Else \quad \Return $\emptyset$;
		\EndIf

	\EndIf

	\If{ $|Y| = 1$ } \Return $Y$;
	\EndIf

	\State split $Y$ into $<Y_1, Y_2>$ such that $|Y_1| = \lceil |Y|/2 \rceil $;

	\State $S_1 \leftarrow FindScope(e,R \cup Y_1, Y_2, true)$;
	\State $S_2 \leftarrow FindScope(e,R \cup S_1, Y_1, (S_1 \neq \emptyset))$;

	\State \Return $S_1 \cup S_2$;

\EndFunction

\end{algorithmic}
\end{algorithm}


An invariant of {\em FindScope} is that the example $e$ violates at least one constraint whose scope is a subset of $R \cup Y$. 
If {\em FindScope} is called with $ask\_query$ = true it asks the user if $e_R$ is positive or not (line 3). If the answer is yes, it removes all the constraints from the bias that reject $e_R$.
Otherwise, it returns the empty set (line 3). {\em FindScope} reaches line 5 only in the case where $e_R$ does not violate any constraint. Hence, because as mentioned above $e$ violates at least one constraint whose scope is a subset of $R \cup Y$, if $Y$ is a singleton, the variable it contains surely belongs to the scope of a constraint that is violated. 
That is because $e_R$ does not violate any constraint (because we have reached at this point), but we know that $e_{R \cup Y}$ is a negative example. 
In this case the function returns $Y$.

If none of the return conditions is satisfied, the set $Y$ is split in two balanced parts (line 6) and the algorithm searches recursively in the sets of variables built using these parts for the scope of a violated constraint, in a logarithmic number of steps (lines 7-9).
Function {\em FindScope} posts partial queries to the user until it finds the scope of a constraint that is violated. A potential deficiency is the fact that if a question to the user violates, say 3 constraints from $B$, and the answer was negative, then after removing some variables from $Y$, if the partial query is  still violating 3 constraints from $B$, {\em FindScope} will ask the user to classify the partial query again. However, there is no need for this because it is certain that the partial query will still be classified as negative. In Section~\ref{findScope2} we propose a fix to this problem.

After the system has located the scope of a violated constraint, it calls function {\em FindC} (Algorithm \ref{alg:findc}) to find the violated constraint. {\em FindC} asks partial queries to the user in the scope returned by {\em FindScope} to locate the violated constraint.
Function {\em FindC} takes as parameters $e$ and $Y$, with $e$ being the negative example in which {\em FindScope} found that there is a violated constraint from the target network $C_T$, and $Y$ being the scope of that constraint.


\begin{algorithm}
\caption{{\em FindC}}\label{alg:findc}
\begin{algorithmic}[1]

\Require $e$, $Y$ ($e$: the example, $Y$: The scope to search)
\Ensure $c$ : a constraint in $C_T$

\Function{FindC}{$e$, $Y$}	

	\State $B \leftarrow B \setminus \{c_Y 	\mid C_L \models c_Y \}$;

	\State $ \Delta \leftarrow \kappa_B(e_Y)$;

	\If{ $\Delta = \emptyset$ } \Return $\emptyset$;
	\EndIf

	\While {true}

		\State Generate $e^{\prime}$ in $D^Y$ accepted by $C_L$ and rejected by $\Delta$, with $\kappa_\Delta(e_Y) \neq |\Delta|$;

		\If{ $e^{\prime} = nil$ }
			\If{ $\Delta = \emptyset$ } \Return $\emptyset$;
			\Else \quad \Return random $c$ in $\Delta$ ;
			\EndIf

		\EndIf

		\If{ $ASK(e^{\prime})$ = ``yes'' } 
			\State $B \leftarrow B  \setminus  \kappa_B(e^{\prime}) $;
			\State $\Delta \leftarrow \Delta  \setminus  \kappa_\Delta(e^{\prime}) $;
		\Else \quad $\Delta \leftarrow \kappa_\Delta(e^{\prime}) $;
		\EndIf

	\EndWhile

\EndFunction

\end{algorithmic}
\end{algorithm}


In more detail, {\em FindC} first removes from the bias the constraints with scope $Y$ that are implied by the learned network $C_L$ (line 2). Next, set $\Delta$ is initialized to the candidate constraints, i.e. the constraints from $B$ with scope $Y$ that are violated by $e$ (line 3). If there are no candidate constraints then the empty set is returned (line 4) resulting in a collapse for QuAcq. In line 5 it enters its main loop in which it posts partial queries to the user. In line 6, a partial example $e^{\prime}$ is generated that is accepted by $C_L$ and rejected by $\Delta$ but not by all of its constraints. This is done to reduce the number of the candidate constraints whatever the answer of the user may be. If no such example exists (line 7), this means
that any of these constraints could be in $C_L$, so one constraint is randomly returned, except if $\Delta$ is empty (lines 8-9). If an example was found then it is posted as a query to the user (line 10). If the answer of the user is ``yes'' then all constraints rejecting it are removed from $B$ and $\Delta$ (lines 11-12), otherwise all constraints accepting it are removed from $\Delta$ (line 13).

Another version of the {\em FindC} function which fixes a problem of Algorithm \ref{alg:findc} is described in~\cite{bessiere2016new}. Namely, in case the target constraint network contains two constraints with scopes $S$ and $S^{\prime}$ such that $S \subset S^{\prime}$ then Algorithm \ref{alg:findc} is not correct. This is because if an example is classified as negative, then when {\em FindC} removes from $\Delta$ all the constraints accepting it, at line 13, it does this under the assumption that a constraint in the scope it currently searches violates the example. 
However, the example may have been rejected because of a constraint in a subscope and not by a constraint in the current scope that is searched. In this case, the constraints accepting the example should not be removed from $\Delta$, but Algorithm \ref{alg:findc} will remove them. The {\em FindC} function introduced in~\cite{bessiere2016new} fixes this problem.
These versions of {\em FindC} can deal only with normalized constraint networks. In order to handle non-normalized constraints a different version of {\em FindC} should be used. However, developing such a method is not within the scope of this paper. Thus, we assume that the target constraint network is normalized, following the literature. 

\subsection{MultiAcq}

Given that there may be several constraints from the target network that are violated by a generated membership query, it is very likely that the system can learn more information from a negative generated example (i.e. acquire more constraints). This is what MultiAcq tries to do, learning a maximum number of constraints from each negative example \cite{arcangioli2016multiple}. This is done by using function {\em FindAllScopes}. After a negative answer to a query $e_Y$, it posts a series of partial queries by removing one variable from $Y$ for each query. In case in all of these calls the example is positive then $Y$ is the scope of a violated constraint. Otherwise, it focuses on all the negative partial queries to find minimal scopes. 

In more detail, MultiAcq (see Algorithm \ref{alg:multiacq}) takes as input a bias $B$ and returns a constraint network $C_L$ equivalent to the target network $C_T$ like QuAcq does. It iteratively generates an example like QuAcq and then the function {\em FindAllScopes} is called to learn a maximum number of constraints violated by the specific example (line 8). Before the call of {\em FindAllScopes}, it initializes the set $MSes$, in which it will store the minimal scopes found, to the empty set.


\begin{algorithm}
\caption{MultiAcq: Multiple Acquisition}\label{alg:multiacq}
\begin{footnotesize}
\begin{algorithmic}[1]

\Require $B$, $X$, $D$ ($B$: the bias, $X$: the set of variables, $D$: the set of domains)
\Ensure $C_L$ : a constraint network

\State $C_L \leftarrow \emptyset$;

\While {true}

	\If{ $sol(C_L) = \emptyset$ } \Return ``collapse'';
	\EndIf

	\State Generate $e$ in $D^X$ accepted by $C_L$ and rejected by $B$;

	\If{ $e$ = nil } \Return ``$C_L$ converged'';
	\Else
		\State $MSes \leftarrow \emptyset$;
		\State $FindAllScopes(e, X, MSes)$

		\Foreach {$Y \in MSes$}

			\State $c_Y \leftarrow FindC(e, Y)$;

			\If{ $c_Y$ = nil } \Return ``collapse'';
			\Else \quad $C_L \leftarrow C_L \cup \{c_Y\} $;
			\EndIf

		\EndFor

	\EndIf

\EndWhile

\end{algorithmic}
\end{footnotesize}
\end{algorithm}


The recursive function {\em FindAllScopes} (Algorithm \ref{alg:findallscopes}) takes as input a complete example $e$, a subset of variables $Y$ (equal to $X$ for the first call) and the set of minimal scopes already found (the empty set in the first call). If $Y$ is not a minimal scope already found (line 1) and does not contain a minimal scope already learned (line 3) and $e_Y$ contains at least one violated constraint from the bias (line 2), the system asks the user to classify the (partial) example $e_Y$ (line 4). If the answer is ``yes'' the constraints from $B$ violating the example are removed (line 5) and false is returned (line 6), as $Y$ does not contain any minimal scope. If the answer is ``no'', it means that there still exist violated constraints from $C_T$ in $Y$. So then {\em FindAllScopes} is called on each subset of $Y$ built by removing one variable from $Y$ (lines 8-9). If in all of these calls the example is positive then $Y$ is the scope of a violated constraint and it is added to the set $MSes$ (line 10). Then {\em FindAllScopes} returns true, as it has found a minimal scope.
Function {\em FindC} is then called by MultiAcq to find the constraint(s), like in QuAcq.


\begin{algorithm}
\caption{{\em FindAllScopes}}\label{alg:findallscopes}
\begin{footnotesize}
\begin{algorithmic}[1]

\Require $e$, $Y$, $MSes$ ($e$: an example, $Y$: a set of variables, $MSes$: the set of minimal scopes)
\Ensure a boolean : if $e_Y$ contains a minimal scope

	\If{$Y \in MSes$} \Return true;
	\EndIf

	\If{ $k_B(e_Y) = \emptyset$ } \Return false;
	\EndIf

	\If{ $\nexists M \in MSes$ $|$ $M \subset Y$}

		\If{ $ASK(e_Y)$ = ``yes'' } 
			\State $B \leftarrow B  \setminus  \kappa_B(e_R) $;
			\State \Return false;
		\EndIf

	\EndIf

	\State $flag \leftarrow$ false;

	\Foreach {$x_i \in Y$}
		\State $flag \leftarrow FindAllScopes(e, Y \setminus \{x_i\}, MSes) \lor flag$  
	\EndFor

	\If{ $\neg flag$ } $MSes \leftarrow MSes \cup \{Y\}$;
	\EndIf

	\State \Return true;

\end{algorithmic}
\end{footnotesize}
\end{algorithm}


So, instead of focusing on the scope of only one constraint, MultiAcq learns all the constraints of the target network violated by a generated example. However, a disadvantage of MultiAcq is that it needs a linear number of queries in the size of the example to learn a constraint.

Now, we illustrate the behaviour of QuAcq and MultiAcq using our running example from Section~\ref{sec:inter}.

\begin{example}
\label{ex:quacq}

Recall that the vocabulary $(X,D)$ given to the system is $X = \{ x_1, ..., x_8 \}$ and $D = \{D(x_1), . . . , D(x_8)\}$ with $D(x_i) = \{ 1, ... , 8 \}$, the target network $C_T$ is the set $\{ \neq_{12}, \neq_{13}, \neq_{34} \}$ and $B = \{ \neq_{ij}$ $\mid$ $1 \leq i < 8 \land i < j \leq 8\}$. Assume that the example generated at line 4 of QuAcq or MultiAcq is $e = \{1,1,1,2,3,4,5,6\}$. 

QuAcq will directly post it as a query to the user. The answer will be ``no'' as it violates the constraints $\{ \neq_{12}, \neq_{13} \}$ from the target network. Next, {\em FindScope} is called to find the scope of a violated constraint. Table~\ref{ta:ex-quacq} shows the recursive calls of {\em FindScope}. A dash (-) in columns $e_R$ and {\em ASK} means that no query is posted to the user, due to the condition at line 2. Also recall that queries are always only on the variables in $R$. As we can see, after 4 queries, {\em FindScope} will find the scope $\{x_1, x_2\}$. Then {\em FindC} will immediately return the constraint $\{\neq_{12}\}$ as it is the only constraint in $B$ with this scope. After the constraint is added to $C_L$, QuAcq will go back to line 3 and as no collapse occurs, it will generate a new example.

\begin{table}
\centering
\caption{Recursive calls of {\em FindScope} in Example~\ref{ex:quacq}}
\label{ta:ex-quacq}
{
\resizebox{\textwidth}{!}{%
\begin{tabular}{ |l|l|l|l|c|c|  }
\hline
{\em call} & $R$ & $Y$ & $e_R$ & {\em ASK} & {\em return} \\
\hline
0 & $\emptyset$ & $x_1,x_2,x_3,x_4,x_5,x_6,x_7,x_8$ & - & - & $\{x_1,x_2\}$ \\
1 & $x_1, x_2, x_3, x_4$ & $x_5,x_6,x_7,x_8$ & $\{1,1,1,2,-,-,-,-\}$ & ``no'' & $\emptyset$ \\
2 & $\emptyset$ & $x_1, x_2, x_3, x_4$ & - & - & $\{x_1,x_2\}$ \\
2.1 & $x_1, x_2$ & $x_3, x_4$ & $\{1,1,-,-,-,-,-,-\}$ & ``no'' & $\emptyset$ \\
2.2 & $\emptyset$ & $x_1,x_2$ & - & - & $\{x_1,x_2\}$ \\
2.2.1 & $x_1$ & $x_2$ & $\{1,-,-,-,-,-,-,-\}$ & ``yes'' & $\{x_2\}$ \\
2.2.2 & $x_2$ & $x_1$ & $\{-,1,-,-,-,-,-,-\}$ & ``yes'' & $\{x_1\}$ \\
\hline
\end{tabular}
}
}
\end{table}

After generating the same example $e$, MultiAcq will directly give it to the function {\em FindAllScopes}. Table~\ref{ta:ex-multiacq} presents the trace of its recursive calls. After 8 queries it will find the scopes of both constraints from $C_T$. It will need one more query to remove the constraint $\{\neq_{23}\}$ from the bias and as no other constraint from $B$ is violated it will return. {\em FindC} will return immediately each of the constraints $\{\neq_{12}\}$ and $\{\neq_{13}\}$ as they are the only ones in $B$ with these scopes.


\begin{table}
\centering
\caption{Recursive calls of {\em FindAllScopes} in Example~\ref{ex:quacq}. }
\label{ta:ex-multiacq}
{
\resizebox{\textwidth}{!}{%
\begin{tabular}{ |l|l|l|c|c|c|  }
\hline
{\em call} & $Y$ & $e_Y$ & {\em ASK} & $MSes$ & {\em return} \\
\hline
0 & $x_1,x_2,x_3,x_4,x_5,x_6,x_7,x_8$ & $\{1,1,1,2,3,4,5,6\}$ & ``no'' & $\emptyset$ & true\\
1 & $x_1,x_2,x_3,x_4,x_5,x_6,x_7$ & $\{1,1,1,2,3,4,5,-\}$ & ``no'' & $\emptyset$ & true\\
1.1 & $x_1,x_2,x_3,x_4,x_5,x_6$ & $\{1,1,1,2,3,4,-,-\}$ & ``no'' & $\emptyset$ & true\\
1.1.1 & $x_1,x_2,x_3,x_4,x_5$ & $\{1,1,1,2,3,-,-,-\}$ & ``no'' & $\emptyset$ & true\\
1.1.1.1 & $x_1,x_2,x_3,x_4$ & $\{1,1,1,2,-,-,-,-\}$ & ``no'' & $\emptyset$ & true\\
1.1.1.1.1 & $x_1,x_2,x_3$ & $\{1,1,1,-,-,-,-,-\}$ & ``no'' & $\emptyset$ & true\\
1.1.1.1.1.1 & $x_1,x_2$ & $\{1,1,-,-,-,-,-,-\}$ & ``no'' &  $\{\{x_1,x_2\}\}$ & true\\
1.1.1.1.1.1.1 & $x_1$ & $\{1,-,-,-,-,-,-,-\}$ & - & $\{\{x_1,x_2\}\}$ & false\\
1.1.1.1.1.1.2 & $x_2$ & $\{-,1,-,-,-,-,-,-\}$ & - & $\{\{x_1,x_2\}\}$ & false\\
1.1.1.1.1.2 & $x_1,x_3$ & $\{1,-,1,-,-,-,-,-\}$ & ``no'' & $\{\{x_1,x_2\},\{x_1,x_3\}\}$ & true\\
1.1.1.1.1.2.1 & $x_1$ & $\{1,-,-,-,-,-,-,-\}$ & - &  $\{\{x_1,x_2\},\{x_1,x_3\}\}$ & false\\
1.1.1.1.1.2.2 & $x_3$ & $\{-,-,1,-,-,-,-,-\}$ & - &  $\{\{x_1,x_2\},\{x_1,x_3\}\}$ & false\\
1.1.1.1.1.3 & $x_2,x_3$ & $\{-,1,1,-,-,-,-,-\}$ & ``yes'' &  $\{\{x_1,x_2\},\{x_1,x_3\}\}$ & false\\
\hline
\end{tabular}
}
}
\end{table}

\end{example}

In addition to the above, which are described in the relevant papers, QuAcq and MultiAcq take some extra steps during query generation\footnote{Personal communication with the authors of the algorithms.}. We detail these in  Section~\ref{sec:query}. 


\section{Efficient Multiple Constraint Acquisition}
\label{sec:mquacq}

As explained, the main difference between QuAcq and MultiAcq is that the latter tries to find multiple constraints that are violated once a query is classified as negative. However, MultiAcq needs a linear number of queries in the size of the example to locate the scope of each violated constraint. In contrast, QuAcq requries a logarithmic number of queries but finds only one violated constraint. Our intuition was to merge the idea of learning a maximum number of constraints from each generated negative example with the QuAcq reasoning of learning each constraint in a logarithmic number of steps.
Our resulting new algorithm, called Multi-QuAcq (MQuAcq for short), needs a logarithmic number of queries to discover each violated constraint, achieving the benefits of both QuAcq and MultiAcq.


\subsection{Multi-QuAcq description}

MQuAcq (Algorithm \ref{alg:all}) is an active learning algorithm which is based on QuAcq and extends it by incorporating the basic idea of MultiAcq. The main difference between QuAcq and MQuAcq is the fact that QuAcq finds one explanation (constraint) of why the user classified an example as negative, whereas MQuAcq tries to learn all the violated constraints. This is done by calling function {\em FindScope} (Algorithm~\ref{alg:findscope}) iteratively, while reducing the search space by removing variables from the scopes already found. 
The main difference between MQuAcq and MultiAcq is that the former uses the QuAcq search method to find each scope through function {\em FindScope}, and in this way avoids some redundant searches (which can be very time-consuming) as well as queries that MultiAcq makes with function {\em FindAllScopes}. Besides this, there are several other differences, particularly on how the algorithms operate to locate irredundant queries after learning a constraint from a generated negative example.
As a result, our proposed algorithm has a better complexity bound in terms of the number of queries. 

MQuAcq finds all the violated constraints via the function {\em FindAllCons}. The main idea is that after finding a constraint $c$, using QuAcq's reasoning, we exploit the fact that for any other violated constraint $c^{\prime} \in C_T$, we have $scope(c) \setminus scope(c^{\prime}) \neq \emptyset$. This is because otherwise $scope(c)$ would not be a minimal scope. So, {\em FindAllCons} recursively acquires all the violated constraints from $e_{Y \setminus \{ x \}}$, for each $x \in scope(c)$. Hence, the reasoning of QuAcq is recursively used in these partial examples, in order to find multiple constraints, with the benefit (inherited from QuAcq) of finding the scope of each constraint with a logarithmic complexity.


\begin{algorithm}
\caption{The MQuAcq Algorithm}\label{alg:all}
\begin{algorithmic}[1]

\Require $B$, $X$, $D$ ($B$: the bias, $X$: the set of variables, $D$: the set of domains)
\Ensure $C_L$ : a constraint network

\State $C_L \leftarrow \emptyset$;

\State $collapse \leftarrow$ false; 

\While {true}

	\If{ $sol(C_L) = \emptyset$ } \Return ``collapse'';
	\EndIf

	\State Generate $e$ in $D^X$ accepted by $C_L$ and rejected by $B$;

	\If{ $e$ = nil } \Return ``$C_L$ converged'';
	\EndIf

	\State $FindAllCons(e, X, \emptyset)$;

	\If{ $collapse = $ true } \Return ``collapse'';
	\EndIf

\EndWhile

\end{algorithmic}
\end{algorithm}


MQuAcq starts by initializing the $C_L$ network to the empty set (line 1) and the global variable $collapse$ to false (line 2). This variable is used within function {\em FindAllCons} as explained below. Next, the algorithm enters its main loop (line 3).
If $C_L$ is unsatisfiable, the algorithm collapses (line 4). Otherwise, a complete assignment $e$ is generated (line 5), satisfying $C_L$ and violating at least one constraint in $B$. This step is explained in detail in Section~\ref{sec:query}.
If such an example does not exist then we have converged (line 6). Otherwise, the function {\em FindAllCons} is called to find all the constraints that are violated by the example $e$ (lines 7). If {\em findAllCons} has detected a collapse then the algorithm terminates (line 8).

The recursive function {\em FindAllCons} is presented in Algorithm \ref{alg:allcons}. It takes as parameters an example $e$, a set of variables $Y$ and a set $Scopes$, which contains the scopes of the violated constraints on $e_Y$ already learned. It returns the set $\operatorname{NScopes}$ consisting of the scopes of the constraints acquired. {\em FindAllCons} adds to $C_L$ all the constraints that are violated by $e$ in $Y$.
The sets $Scopes$ and $\operatorname{NScopes}$ are used to store all the scopes of the constraints that have been found in $e_Y$ to avoid searching and asking partial queries with the scope of a violated constraint that has been already learned. Specifically, the set $Scopes$ stores the scopes of the constraints learned before the current call of the function. On the other hand, the set $\operatorname{NScopes}$ stores the scopes of the constraints learned from the current call (or any sub-call).
For example, if we have acquired 3 constraints from $e_Y$, 2 from a previous call and 1 from the current call of {\em FindAllCons}, then the set $Scopes$ will contain the scopes of the 2 constraints previously learned and the set $\operatorname{NScopes}$ will contain the scope of the constraint learned from the current call.

Our proposed approach is to search for partial queries in the given example that do not contain any constraint already found, so that the answer will not be predictable. To achieve this, from each scope $S$ already found we make $|S|$ partial examples, one for each variable $x_i \in S$, with each such example involving variables $Y^{\prime} = Y \setminus \{x_i\}$. 
When a partial example that violates no constraint already learned but at least one from $B$ is found, {\em FindAllCons} uses  {\em FindScope}, as in QuAcq, to learn a constraint from $C_T$.


\begin{algorithm}
\caption{FindAllCons}\label{alg:allcons}
\begin{algorithmic}[1]

\Require $e, Y, Scopes$ ($e$: the example, $Y$: set of variables, $Scopes$: a set of scopes already learned)
\Ensure $\operatorname{NScopes}$ : the set of scopes learned

\Function{FindAllCons}{$e$, $Y$, $Scopes$}	

	\If{ $collapse = $ true } \Return $\emptyset$;
	\EndIf

	\If { $\nexists scope(c) \neq S$ $|$ $c \in \kappa_{B \setminus C_L}(e_Y) \land S \in Scopes $ } \Return $\emptyset$;
	\EndIf

	\State $\operatorname{NScopes} \leftarrow \emptyset$;

	\If{ $Scopes \neq \emptyset$ }

		\State pick an $S \in Scopes$;

		\Foreach{ $x_i \in S$ }

			\State $\operatorname{NScopes} \leftarrow \operatorname{NScopes} \cup FindAllCons(e, Y \setminus \{x_i\}, \operatorname{NScopes} \cup (Scopes \setminus \{S\} ) )$;

		\EndFor

	\Else

		\If{ ASK($e_Y$) = ``yes'' }  $B \leftarrow B  \setminus  \kappa_B(e_Y) $;
		\Else

			\State $scope \leftarrow FindScope( e, \emptyset, Y, false ) $;

			\State $c \leftarrow FindC(e, scope)$;

			\If{ $c$ = nil }

				\State $collapse \leftarrow$ true;

				\State \Return $\emptyset$;

			\Else \quad $C_L \leftarrow C_L \cup \{c\} $;
			\EndIf

			\State $\operatorname{NScopes} \leftarrow \operatorname{NScopes} \cup \{scope\}$;

			\State $\operatorname{NScopes} \leftarrow \operatorname{NScopes} \cup FindAllCons(e, Y, \operatorname{NScopes})$;

		\EndIf

	\EndIf

	\State \Return $\operatorname{NScopes}$;

\EndFunction

\end{algorithmic}
\end{algorithm}

{\em FindAllCons} starts by checking if collapse has occurred. If this is the case the empty set is returned (line 2). Then  it checks if there exists any violated constraint in $B$ to learn, with a scope different to those of the constraints already acquired. If no constraint that can be learned exists, it is implied that ASK($e_Y$) = ``yes''. Thus, again we return the empty set (line 3) because we assume that the bias is expressive enough to learn a $C_L$ equivalent to the target network $C_T$. This check is important because as the recursive calls to {\em FindAllCons} remove variables from $Y$ (as explained below), we may end up in a case where $e_Y$ is surely positive and no search for a violated constraint is needed. This is because if ASK($e_Y$) = ``yes'' then for every $Y^{\prime} \subseteq Y$ we surely have ASK($e_{Y^{\prime}}$) = ``yes''. With this check the algorithm avoids a lot of redundant searches, reducing the number of nodes in the tree of recursive calls, and also avoids asking redundant queries.
In the case that neither of the two conditions is satisfied, {\em FindAllCons} will continue. At line 4, the set $\operatorname{NScopes}$ is initialized to the empty set.

After that, {\em FindAllCons} checks if the set $Scopes$ is not empty (line 5). If this is the case it means that we have not branched on all the scopes already found and we still have in $e_Y$ the scope of at least one violated constraint.
So we call {\em FindAllCons} recursively on each subset of $Y$ created by removing one of the variables of a scope $S \in Scopes$, removing the scope $S$ in which we branched from the set $Scopes$ given to the recursive calls (lines 6-8). We remove the scope in which we branched as this scope is no longer included in the set of variables given to to the recursive calls. Also we give to the function the set $\operatorname{NScopes} \cup (Scopes \setminus \{S\} )$ as although the set $\operatorname{NScopes}$ is empty at the first call, it may contain scopes found in the next recursive calls.

In the case that the set $Scopes$ is empty (line 9), it means that we have finished with branching and we have a partial example $e_Y$ that does not contain the scope of any violated constraint already learned. Hence, there must exist a partial query $e_Y$ that violates at least one constraint of $B$ (otherwise the algorithm would have returned at line 2) and no violated constraint already found exists in $Y$. Therefore, the system asks again the user to classify the partial example as positive or negative 
(line 10). If the answer is positive then the constraints in $B$ that reject $e$ are removed. Otherwise, function {\em FindScope} is called to find the scope of one of the violated constraints (line 12). {\em FindC} will then select a constraint from $B$ with the discovered scope that is violated by $e_Y$ (line 13). If no constraint is found then the algorithm collapses (line 14-16). Otherwise, the constraint returned by {\em FindC} is added to $C_L$ (line 17) and its scope is added to the set of found scopes (line 18). Then, we call again {\em FindAllCons} to continue searching in the partial examples created by removing the variables of the scope the function has just found.



We now illustrate the behavior of {\em FindAllCons} in a simple problem using our running example.

\begin{example}
\label{ex:mquacq}

Recall that the problem consists of $8$ variables and suppose that the complete example $e = \{1,1,1,1,2,3,4,5\}$ is generated in line 5 of MQuAcq. The constraints from $C_T$ that are violated by $e$ are $\neq_{12}$, $\neq_{13}$ and $\neq_{34}$ (all the constraints from $C_T$ in this case). Table~\ref{ta:ex-mquacq} presents the trace of the recursive calls of {\em FindAllCons}.


\begin{table}
\centering
\caption{Recursive calls of {\em FindAllCons} in Example~\ref{ex:mquacq}}
\label{ta:ex-mquacq}
{
\resizebox{\textwidth}{!}{%
\begin{tabular}{ |l|l|l|c|c|c|  }
\hline
{\em call} & $Y$ & $e_Y$ & {\em ASK} & $Scopes$ & {\em return} \\
\hline
0 & $x_1,x_2,x_3,x_4,x_5,x_6,x_7,x_8$ & $\{1,1,1,1,2,3,4,5\}$ & ``no'' & $\emptyset$ & $\{\{x_1,x_2\}, \{x_3,x_4\}, \{x_1,x_3\}\}$\\
1 & $x_1,x_2,x_3,x_4,x_5,x_6,x_7,x_8$ & $\{1,1,1,1,2,3,4,5\}$ & - & $\{\{x_1,x_2\}\}$ & $\{\{x_3,x_4\}, \{x_1,x_3\}\}$\\
1.1 & $x_2,x_3,x_4,x_5,x_6,x_7,x_8$ & $\{-,1,1,1,2,3,4,5\}$ & ``no'' & $\emptyset$ & $\{\{x_3,x_4\}\}$\\
1.1.1 & $x_2,x_3,x_4,x_5,x_6,x_7,x_8$ & $\{-,1,1,1,2,3,4,5\}$ & - & $\{\{x_3,x_4\}\}$ & $\emptyset$\\
1.1.1.1 & $x_2,x_4,x_5,x_6,x_7,x_8$ & $\{-,1,-,1,2,3,4,5\}$ & ``yes'' & $\emptyset$ & $\emptyset$\\
1.1.1.2 & $x_2,x_3,x_5,x_6,x_7,x_8$ & $\{-,1,1,-,2,3,4,5\}$ & ``yes'' & $\emptyset$ & $\emptyset$\\
1.2 & $x_1,x_3,x_4,x_5,x_6,x_7,x_8$ & $\{1,-,1,1,2,3,4,5\}$ & - & $\{\{x_3,x_4\}\}$ & $\{\{x_1,x_3\}\}$\\
1.2.1 & $x_1,x_4,x_5,x_6,x_7,x_8$ & $\{1,-,-,1,2,3,4,5\}$ & ``yes'' & $\emptyset$ & $\emptyset$\\
1.2.2 & $x_1,x_3,x_5,x_6,x_7,x_8$ & $\{1,-,1,-,2,3,4,5\}$ & ``no'' & $\emptyset$ & $\{\{x_1,x_3\}\}$\\

\hline
\end{tabular}
}
}
\end{table}


In the first call (call 0) to {\em FindAllCons}, $e$ will be posted as a query to the user. After the user answers ``no'', the algorithm will find the constraint $\neq_{12}$ using functions {\em FindScope} and {\em FindC}. Next, {\em FindAllCons} will be called to continue searching for the remaining constraints that violate $e$. In the next call (call 1) we have $Y = X$ and $Scopes = \{\{x_1,x_2\}\}$. As $Scopes \neq \emptyset$, we know that the answer to ASK($e_Y$) (with $e_Y = e_X$) will be ``no''. So, no query is posted to the user and {\em FindAllCons} will be called recursively on each subset of $Y$ built by removing one variable from a scope $S \in Scopes$ (i.e. $\{x_1,x_2\}$ as this is the only one), and removing this scope from the set given to the recursive calls. 
Thus, in the first recursive call (call 1.1) we have $Y^{\prime} = Y \setminus \{x_1\}$ and $Scopes = \emptyset$. This means that we have branched on all scopes found until now. Hence, the example $e_{Y^{\prime}} = \{-,1,1,1,2,3,4,5\}$ will be posted as a query to the user and the constraint $\neq_{34}$ will be learned because it is the only constraint from $C_T$ that is violated by $e_{Y^{\prime} }$. In the next call (call 1.1.1) of {\em FindAllCons} in line 14 no further constraint will be found, as no constraint from $C_T$ is violated. 

So, we go back to the second call of line 5 (call 1.2). We have $Y^{\prime} = Y \setminus \{x_2\}$ and $Scopes = \{\{x_3,x_4\}\}$ (the scope of the constraint learned from the call 1.1). As $Scopes \neq \emptyset$, we have another scope in which we have to branch. Hence, {\em FindAllCons} will be called recursively on each subset of $Y^{\prime}$ built by removing one variable from a scope $S \in Scopes$ (i.e. $\{x_3,x_4\}$ as this is the only one). Also this scope is removed from the set of scopes given to the recursive calls.

In call 1.2.1, we have $Y^{\prime\prime} = Y^{\prime} \setminus \{x_3\}$ and $Scopes = \emptyset$. Because no constraint from $C_T$ is violated by $e_{Y^{\prime\prime}} = \{1,-,-,1,2,3,4,5\}$, the answer from the user will be ``yes'' and the empty set will be returned. In call 1.2.2, we have $Y^{\prime\prime} = Y^{\prime} \setminus \{x_4\}$ and $Scopes = \emptyset$. Thus, the example $e_{Y^{\prime\prime}} = \{1,-,1,-,2,3,4,5\}$ will be posted to the user and then the constraint $\neq_{13}$ will be learned. No further constraint will be found, as no other constraint from $C_T$ is violated by $e_Y$. 

\end{example}

\subsection{Analysis}

In this section we prove the correctness (i.e. soundness and completeness) of MQuAcq. To obtain this proof we first prove some properties of functions {\em FindScope} and {\em FindC}. We also study the complexity of MQuAcq in terms of the number of queries it needs to converge to the target network. 

\begin{lemma}
\label{lemma1}
If ASK($e_Y$) = ``yes'' then for any $Y^{\prime} \subseteq Y$ it holds that ASK($e_{Y^{\prime}}$) = ``yes''.
\end{lemma}

\begin{proof}
We know that for every $Y^{\prime} \subseteq Y$, the set of constraints from $C_T$ that are violated by $e_{Y^{\prime}}$ is a subset of the set of constraints rejecting $e_Y$ 
(i.e. $\kappa_{C_T}(e_{Y^{\prime}}) \subseteq \kappa_{C_T}(e_Y) $). 
Thus, if we know that $\kappa_{C_T}(e_Y) = \emptyset$ (ASK($e_Y$) = ``yes'') then for every $Y^{\prime} \subseteq Y$ it holds that $\kappa_{C_T}(e_{Y^{\prime}}) = \emptyset$ which means that ASK($e_{Y^{\prime}}$) = ``yes''.
\end{proof}

\begin{lemma}
\label{lemma2}
If ASK($e_Y$) = ``no'' then for any $Y^{\prime} \supseteq Y$ it holds that ASK($e_{Y^{\prime}}$) = ``no''.
\end{lemma}

\begin{proof}
We know that for every $Y^{\prime} \supseteq Y$, the set of constraints from $C_T$ that are violated by $e_{Y^{\prime}}$ is a superset of the set of constraints rejecting $e_Y$ 
(i.e. $\kappa_{C_T}(e_{Y^{\prime}}) \supseteq \kappa_{C_T}(e_Y) $). 
Thus, if we know that $\kappa_{C_T}(e_Y) \neq \emptyset$ (ASK($e_Y$) = ``no'') then for every $Y^{\prime} \supseteq Y$ it holds that $\kappa_{C_T}(e_{Y^{\prime}}) \neq \emptyset$ which means that ASK($e_{Y^{\prime}}$) = ``no''.
\end{proof}

Lemmas~\ref{lemma1} and~\ref{lemma2} have also been proved in~\cite{arcangioli2016multiple}, albeit slightly differently.

\begin{prop}
\label{prop:findscope}

Given the assumption that $C_T$ is representable by $B$, if FindScope is given an example $e_Y$ and returns a scope $S$, then there exists a violated constraint $c \in C_T$ with $scope(c) = S$. Also $S$ is a minimal scope.

\end{prop}

\begin{proof} 

Recall that an invariant of {\em FindScope} is that the example $e$ violates at least one constraint whose scope is a subset of $R \cup Y$ (i.e. ASK($R \cup Y$) = ``no''). Also, it reaches line 5 only in the case that $e_R$ does not violate any constraint from $C_T$ (i.e. ASK($e_R$) = ``yes'' from Lemma~\ref{lemma1}). In addition, in {\em FindScope} variables are returned only at line 5, in the case $Y$ is a singleton. Thus, for any $x_i \in S$ we know that ASK($S$) = ``no'' and ASK($S \setminus x_i$) = ``yes''. Hence, $S$ is a scope of a violated constraint from the target network. Also, as we have ASK($S \setminus x_i$) = ``yes' for any $x_i \in S$, it holds that $S$ is a minimal scope.

\end{proof}

\begin{prop}
\label{prop:findc}

Given an example $e$, the scope $Y$ of a violated constraint of $C_T$ and a bias $B$ that can represent $C_T$, FindC will return a constraint $c \in C_T$ with $scope(c) = Y$ under the assumption that $C_T$ does not contain any other constraint with scope $Y' \subseteq Y$.

\end{prop}

\begin{proof} 

$\forall c \in \Delta$ we know that $c$ is violated by $e_Y$ as $c \in \kappa_{B}(e_Y)$ (line 2). Thus, as $B$ can represent $C_T$, if $Y$ is the scope of a violated constraint $c \in C_T$, this constraint is surely included in $\Delta$. Now we will prove that this constraint is never removed from $\Delta$ but the constraints not in $C_T$ are. Constraints are removed from $\Delta$ only at lines 12,13. 
When the user's answer to the generated query is ``yes'' then the constraint that we seek is not violated (Lemma~\ref{lemma1}), so $\Delta \leftarrow \Delta \setminus \kappa_\Delta(e^{\prime}) $ does not remove it. On the other hand, if the user's answer is ``no'' then the constraint $c \in C_T$ that we seek is surely violated as $C_T$ does not contain any other constraint with a scope $Y' \subseteq Y$. Therefore, the operation $\Delta \leftarrow \kappa_\Delta(e^{\prime})$ does not remove it. 

Thus, an invariant of {\em FindC} is that $\Delta$ surely includes a constraint that made the user to classify as negative the example $e_Y$ given to the function, as it is added to $\Delta$ and never removed from it. Hence, if an example accepting some constraints in $\Delta$ and rejecting others cannot be generated at line 6, all the constraints in $\Delta$ are equivalent wrt $C_L$. Thus, whichever among them (if more than one) is returned, this constraint $c$ is surely included in $C_T$.

\end{proof}

\begin{thm} 
\label{correctness}


Given a bias $B$ built from a language $\Gamma$, with bounded arity constraints, and a target network $C_T$ representable by $B$, MQuAcq is correct.

\end{thm}

\begin{proof} 

{\em Soundness}. 
MQuAcq learns constraints only via the function {\em FindAllCons}. {\em FindAllCons} learns a constraint using the function {\em FindC} after finding the scope of the constraint with the function {\em FindScope}. Given the assumption that the user's answers are correct and that the target network $C_T$ is representable by $B$, {\em FindScope} returns the scope of a violated constraint from $C_T$ (Proposition~\ref{prop:findscope}). Also, as {\em FindC} is called with the scope returned from {\em FindScope} and the example classified as negative by the user, it will return a violated constraint $c \in C_T$ (Proposition~\ref{prop:findc}). Thus, {\em FindAllCons} is sound, which means that MQuAcq is sound, as for every constraint $c$ added to $C_L$ it holds that $c \in C_T$.

{\em Completeness}.
MQuAcq learns constraints only via the function {\em FindAllCons}. Thus, if given an example $e_Y$ {\em FindAllCons} can acquire any violated constraint $c \in C_T$ then MQuAcq is complete. That is because MQuAcq iteratively generates examples that violate constraints from $B$ and gives them to {\em FindAllCons}. It stops only if no example can be generated at line 5. If this is the case, it means that the system has converged as $C_L$ agrees with $E$ and for every other network $C \subseteq B$ that agrees with E, we have $sol(C) = sol(C_L)$.
Now we will prove that given an example $e_Y$ function {\em FindAllCons} can acquire any violated constraint $c \in C_T$.

Given the assumption that the target network $C_T$ is representable by $B$, if the condition at line 3 is satisfied then we know that no constraint from $C_T$ can be learned. Also, given an example $e_Y$, with $Y \subseteq X$, we know that if ASK($e_Y$) = ``yes'' then for any $Y^{\prime} \subseteq Y$ it is ASK($e_{Y^{\prime}}$) = ``yes'' (Lemma~\ref{lemma1}). Hence, in this case again no constraint from $C_T$ is violated by $e_Y$.
Thus, if a minimal scope $M$ exists in $Y$ then the condition at line 3 is not satisfied and the answer from the user at line 10 will be ``no'', as $Y \supseteq M$ (Lemma~\ref{lemma2}). Thus, if function {\em FindAllCons} is given an example $e_Y$ violating a constraint from $C_T$, it will reach lines 12-13 to learn the constraint using functions {\em FindScope} and {\em FindC}.
In addition, {\em FindAllCons} will surely search for any minimal scope $M$ in a $Y \supseteq M$. We know that if a constraint with scope $S$ is learned from an example, then {\em FindAllCons} will search for another violated constraint from $C_T$ in $Y \setminus \{x_i\}, \forall x_i \in S$. Thus, in each $e_{Y \setminus \{x_i\}}$ it will search for minimal scopes $M$ such that $x_i \not\in M$ and $M \subseteq Y \setminus \{x_i\}, \forall x_i \in S$. Generalizing this, function {\em FindAllCons} will find all the minimal scopes $M$ such that $S \nsubseteq M$ and $M \subseteq Y$. Hence, it will find all the minimal scopes in $Y$ and then learn the constraints by calling {\em FindC}.

\end{proof}

We now analyse the complexity of MQuAcq in terms of the number of queries it asks to the user. 

\begin{thm} 
\label{complexity}

Given a bias $B$ built from a language $\Gamma$, with bounded arity constraints, and a target network $C_T$, 
MQuAcq uses $O(|C_T| \cdot (log|X| + |\Gamma|))$ queries to find the target network or to collapse and $O(|B|)$ queries to prove convergence.

\end{thm}

\begin{proof} 
Queries are asked to the user in lines 10 of {\em FindAllCons}, 3 of {\em FindScope} and 9 of {\em FindC}.
We know that a scope of a constraint from $C_T$ is found in $O(|S| \cdot log|Y|)$ queries with the function {\em FindScope}, with $|S|$ being the arity of the scope and $|Y|$ the size of the example given to the function \cite{bessiere2013constraint}. As $Y \subseteq X$, {\em FindScope} needs at most $|S| \cdot log|X|$ queries to find the scope of a constraint, because in {\em FindAllCons}, in the worst case, only one constraint from $C_T$ will be violated by any complete example. Also, {\em FindC} needs at most $|\Gamma|$ queries to find a constraint from $C_T$ in the scope it takes as parameter, if one exists \cite{bessiere2013constraint}. If none exists, the system collapses. Hence, in the worst case, the number of queries necessary to find each constraint is 
$O(|S| \cdot log|X|+|\Gamma|)$. Thus, the number of queries for finding all the constraints in $C_T$ or collapsing is at most $C_T \cdot (|S| \cdot log|X|+|\Gamma|)$ which is $O(C_T \cdot (log|X|+|\Gamma|))$ because $|S|$ is bounded.
Convergence is proved when $B$ is empty or contains only redundant constraints. Constraints are removed from $B$ when the answer from the user is ``yes'' in a query. In the case that the example generated by the algorithm in line 5, contains only one violated constraint from $B$, it leads to at least one constraint removal in each query. This gives a total of $O(|B|)$ queries to prove convergence. 
\end{proof}

The complexities of QuAcq and MultiAcq to find the target network are $O(|C_T| \cdot (log|X| + |\Gamma|))$ and $O(|C_T| \cdot (|X| + |\Gamma|))$ respectively. Hence, we achieve the same bound as QuAcq but a better one than MultiAcq, while discovering all the violated constraints from a negative example.


\section{FindScope-2}
\label{findScope2}

We now describe an optimization to function {\em FindScope}, aiming at asking fewer queries to the user by avoiding posting redundant queries. This results in a function we simply call {\em FindScope-2}, which can be used instead of {\em FindScope} either inside QuAcq or inside our proposed algorithm MQuAcq.

Let us first consider a simple example to show a deficiency of {\em FindScope}.

\begin{example}

Consider the behaviour of {\em FindScope} in the simple problem of the Example~\ref{ex:quacq}. The recursive calls of {\em FindScope} are illustrated in Table~\ref{ta:ex-quacq}. The negative example given to {\em FindScope} is $e = \{1,1,1,2,3,4,5,6\}$. The constraints from $B$ that it violates are $\kappa_B(e) = \{\neq_{12}, \neq_{13}, \neq_{23}\}$. After the first call to  {\em FindScope}, $R$ is equal to $x_1, x_2, x_3, x_4$, so the partial example $e_R$ that is then asked to the user is $e_R = \{ 1,1,1,2 \}$. 
As we can see, the constraints from $B$ that are violated are still $\kappa_B(e) = \{\neq_{12}, \neq_{13}, \neq_{23}\}$. Therefore, this partial example is negative as no violated constraint from $B$, that could be included to $C_T$, is removed. Thus, there is no point in posting it to the user.

In addition, given the assumption that the bias is expressive enough to learn $C_T$, in cases where $|\kappa_B(e_R)| = 0$ (i.e. there is no violated constraint in $B$), it is implied that ASK($e_R$) = ``yes''. For example, see the last two queries asked from {\em FindScope} in the current example. The are queries that include only one variable, but the bias does not include any unary constraint. Thus, it is implied that these are positive examples and thus, these queries were redundant.

\end{example}

To avoid such redundant queries made by {\em FindScope}, we modify this function (see Algorithm \ref{alg:findscope2}) adding a check that inspects if the number of violated constraints from the bias is the same as in the last query asked. This is implemented using a global variable $rej$ to store this number. This check is done in line 3. If this is the case, it is implied that the answer will still be no and therefore we return the empty set. Before the first call to {\em FindScope}, $rej$ must be initialized to the number of constraints from $B$ that are violated by the complete query. 


\begin{algorithm}
\caption{FindScope-2}\label{alg:findscope2}
\begin{algorithmic}[1]

\Require $e$, $R$, $Y$, $ask\_query$ ($e$: the example, $R$,$Y$: sets of variables, $ask\_query$: boolean)
\Ensure $Scope$ : a set of variables, the scope of a constraint in $C_T$

\Function{FindScope-2}{$e$, $R$, $Y$, $ask\_query$}	

	\If{ $ask\_query \land |\kappa_B(e_R)| > 0$} 

		\If{ $rej \neq |\kappa_B(e_R)|$ } 

			\If{ ASK($e_R$) = ``yes'' } $B \leftarrow B  \setminus  \kappa_B(e_R) $;
			\Else 

				\State $ rej \leftarrow |\kappa_B(e_R)|$;

				\State \Return $\emptyset$;

			\EndIf

		\Else \quad \Return $\emptyset$;
		\EndIf

	\EndIf

	\If{ $|Y| = 1$ } \Return $Y$;
	\EndIf

	\State split $Y$ into $<Y_1, Y_2>$ such that $|Y_1| = \lceil |Y|/2 \rceil $;

	\State $S_1 \leftarrow FindScope-2(e,R \cup Y_1, Y_2, true)$;
	\State $S_2 \leftarrow FindScope-2(e,R \cup S_1, Y_1, (S_1 \neq \emptyset))$;

	\State \Return $S_1 \cup S_2$;

\EndFunction

\end{algorithmic}
\end{algorithm}


As a further improvement to {\em FindScope}, in cases where $|\kappa_B(e_R)| = 0$ no query is asked to the user as is implied that ASK($e_R$) = ``yes''. So another check is performed in line 2. If the bias is not expressive enough to learn $C_T$, the system will collapse later, because it will not find any constraint to learn.


\subsection{FindScope-2 analysis}
\label{findScope2-analysis}

\begin{prop}
\label{prop:findscope-2}

Given the assumption that $C_T$ is representable by $B$, if {\em FindScope-2} is given an example $e_Y$ and returns a scope $S$, then there exists a violated constraint $c \in C_T$ with $scope(c) = S$. Also $S$ is a minimal scope.

\end{prop}

\begin{proof} 

Let us first prove that the invariant of {\em FindScope} that the example $e$ violates at least one constraint whose scope is a subset of $R \cup Y$ (i.e. ASK($R \cup Y$) = ``no'') applies to {\em FindScope-2} as well. The check added at line 2 does not affect this property, as in the case $|\kappa_B(e_R)| = 0$ we know that for any $Y^\prime \subseteq Y$ it holds that ASK($e_{Y^\prime}$) = ``yes'' (Lemma~\ref{lemma1}). Thus, it reaches line 9 only in the case that $e_R$ does not violate any constraint from $C_T$. Focusing on the check added at line 3, in the case $rej = \kappa_B{e_R}$ it returns $\emptyset$, because we know that the answer of the user would be negative.

Thus, for the same reason as in {\em FindScope}, for any $x_i \in S$, we know that ASK($S$) = ``no'' and ASK($S \setminus x_i$) = ``yes''. Hence, $S$ is a scope of a violated constraint from the target network. Also, as we have ASK($S \setminus x_i$) = ``yes' for any $x_i \in S$, it holds that $S$ is a minimal scope.

\end{proof}

\begin{prop}
\label{prop:findscope-2_compl}

Given a negative example $e_Y$, {\em FindScope-2} posts $\Theta(|S| \cdot log|Y|)$ queries in order to find the scope of a violated constraint, in the worst case.

\end{prop}

\begin{proof}

The number of queries posted by {\em FindScope-2} is equal to the number of nodes of the tree of recursive calls in the worst case, in which a query is posted in each node. Now we will find the number of nodes of this tree in this case.
The branches of the tree that will find a variable in $S$ will have $log|Y|$ depth. The tree of recursive calls to {\em FindScope-2} will have $|S|$ such branches. Thus, in such a case the number of nodes of these branches is $n = |S| \cdot log|Y|$. 
{\em FindScope-2} either makes two recursive calls in each call or it returns $\emptyset$. Thus, for each node on a branch that finds a variable in $S$, we have one sibling that is either a leaf, or starts another branch that will find another variable in $S$. Hence, the leaves that do not return a variable in $S$ will be $n - |S|$. As a result, the number of nodes will be $2 \cdot |S| \cdot log|Y| - |S|$. As a result, {\em FindScope-2} posts $\Theta(|S| \cdot log|Y|)$ queries to the user in the worst case.

\end{proof}

\section{Query Generation in Constraint Acquisition}
\label{sec:query}

In constraint acquisition we would ideally want every generated query to contain as much ``information'' to be learned as possible. That is, we would like to generate queries that violate as many constraints as possible.
To acquire this information we want a constraint to be violated and then, via the user's answers, the algorithm will decide either to learn the constraint or to remove it from the bias. We want to maximize the constraints from $B$ that reject the generated example because this can reduce the number of queries required to converge. That is because after a positive query, all the constraints rejecting the example are removed from $B$. MQuAcq, like QuAcq, has a worst case complexity in terms of the number of queries in $O(|C_T| \cdot log|X| + O(|B|)$. Hence, it is desirable to prune from $B$ the constraints that are not included in $C_T$ with as few queries as possible. Ideally we want each positive query to violate a maximum number of constraints from $B$. With this objective, as proved in~\cite{bessiere2013constraint}, we can bridge the gap to $O(|C_T|log|X|)$ queries for some simple languages, and avoid needing a number of queries up to $|B|$ to prove convergence.


The standard technique for query generation in constraint acquisition systems, such as QuAcq and MultiAcq, is based on the following basic idea: find a solution of the learned constraint network ($C_L$) that violates a maximum number of constraints from the bias. Although the query generation step is a very important step of the acquisition process, it is not described in detail in the literature. Here we focus on query generation and explain it in detail for the first time. We then propose heuristics to enhance its efficiency. 

Query generation it typically viewed as an optimization problem that includes both {\em hard} and {\em soft} constraints. In general, a hard constraint represents a requirement that cannot be violated. All hard constraints must be mandatorily satisfied in a solution. Soft constraints are used to formalize desired properties, preferences that should be satisfied as much as possible.

Query generation can be modelled in this way by considering the constraints from $C_L$ as {\em hard} constraints. The {\em soft} constraints can either (and equivalently) be the constraints from the bias or their complement (the set $\{\neg c $ $|$ $c \in B \setminus C_L $\}). In case the soft constraints are the complement of the constraints from $B$, the objective is to maximize the number of such constraints that are satisfied. Otherwise, the objective is to maximize the number of constraints from $B$ that are violated. 

As CP solvers, whether they can handle soft constraints or not, cannot express this objective naturally, we reformulate the problem as maximizing the the number of satisfied constraint negations.
Hence, we view the problem of query generation as a MAX-CSP that includes hard constraints (the ones from $C_L$) and soft ones (the complement of the constraints in $B$). 
Thus, this is a simple case of (unweighted) MAX-CSP~\cite{khanna2001approximability}, with the requirement that hard constraints must be satisfied, giving all the soft constraints the same importance and the goal is to maximize the number of the satisfied soft constraints.
In the case of query generation there is also another requirement, that is, at least one soft constraint must be satisfied. This is because we want to generate an irredundant query, in which we do not already know the answer.

This is also the approach taken by QuAcq and MultiAcq. When solving this optimization problem, both these algorithms try to find a solution that satisfies all the constraints in $C_L$ and maximizes the satisfaction of the complementary constraints from $B$. This is known as the {\em max} heuristic \cite{bessiere2013constraint,arcangioli2016multiple}.
Except form the above, there are some extra steps in the query generation. This is because the generation of a query that maximizes the violated constraints from $B$ is an NP-hard problem, and therefore may be very time-consuming. 
The query generation process (line 4 of QuAcq and MultiAcq, line 5 of MQuAcq and line 6 of {\em FindC}) is presented in Algorithm~\ref{alg:irr}. We denote the process described above as $QGen(C_h,C_s)$, with $C_h$ being the set of hard constraints and $C_s$ the set of soft constraints.


\begin{algorithm}
\caption{IrrGen: Generate an Irredundant Query}\label{alg:irr}
\begin{algorithmic}[1]

\Require $C_L$, $B$, $V$, $D$  ($C_L$: the learned network, $B$: the bias, $V$: the variables, $D$: the domains of the variables)
\Ensure $q$: an irredundant query

	\State $C_s \leftarrow \{ \neg c $ $|$ $c \in B \setminus C_L  \}$;

	\State $e \leftarrow QGen(C_L, C_s)$;

	\If{ $e \neq nil$ } \Return $e$;
	\Else

		\ForAll{ $c \in C_s$} 
			\State Generate $e$ in $sol(C_L \cap c)$;
			\If{ $e \neq nil$ } \Return $e$;
			\EndIf
		\EndFor

	\EndIf

\State \Return nil;

\end{algorithmic}
\end{algorithm}


First the set $C_s$ containing the soft constraints is initialized (line 1). Next (line 2), an example is generated that is a solution to the learned network $C_L$ and satisfies as many constraints as possible from $C_s$, (i.e. violating as many as possible from the bias).
There are two cutoffs imposed in {\em QGen}, to make sure that the query generation will run in acceptable time. We denote them as $cut_{min}$ and $cut_{max}$.  If the query generator has found a query violating at least one constraint and the first cutoff ($cut_{min}$) is triggered, the best query found is returned. If not, it tries until the maximum time (defined by $cut_{max}$) has been reached. 

If an irredundant query is found then it is returned (line 3). In case no example is found within this time limit then the system tries again, taking the constraints in $B$ one by one (lines 5-8). That is, for each constraint $c \in C_s$, it tries to find a solution of $C_L$ satisfying $c$. The second cutoff of {\em QGen} ($cut_{max}$) is again used for this process. This is done until a query violating at least one constraint from $B$ is found. However, setting any time limit to the query generation process means that the algorithm may reach {\em premature convergence}. That is because it is quite likely that no solution to $C_L$ that violates some constraints from $B$ is found within the time limit, at some point of the algorithm's execution, meaning that it has not been proved that $sol(C) = sol(C_L)$ for every other network $C \subseteq B$ that agrees with $E$.
As a result, all the algorithms suffer from this problem. This is something that has only very recently been pointed out \cite{addi2018time,mquacq}.

\section{Heuristics for boosting query generation}
\label{sec:heur}

In this section we propose heuristics to improve the performance of constraint acquisition systems. In Section~\ref{sec:par} we propose a heuristic for the generation of partial queries by the algorithms. In Sections~\ref{sec:var} and~\ref{sec:value} we propose heuristics for value and variable ordering when trying to generate queries. 

\subsection{Exploiting partial queries}
\label{sec:par}

Let us first note that although both QuAcq and MultiAcq allow for the use of partial queries to focus on the violated constraint(s) after an example has been classified as negative, they both always aim to generate complete examples. 
However, as explained, generating a complete example requires finding a complete variable assignment that satisfies all constraints in $C_L$ and violates at least one constraint in $B$. Given that this is an NP-complete problem, the process can be very time-consuming, especially as the size of $C_L$ grows and the size of $B$ shrinks, i.e. when approaching convergence. 

Experimental results that we have obtained with both QuAcq and MultiAcq demonstrate that when no time limit to the query generation process is set then both algorithms can take several minutes (more than 30 minutes) to generate a query as convergence is approached, even for small problems such as the 9x9 Sudoku. 
This of course is unacceptable from the user's point of view, and therefore a time limit is necessary for the practical application of the algorithms. However, setting a time limit to the query generation process means that the algorithm may reach {\em premature convergence}, as explained before.

Another relevant issue is that of proving convergence in problems that contain redundant constraints. As the system cannot always know beforehand if some of the constraints in the bias are redundant, proving that no solution of $C_L$ violates at least one constraint in $B$ can be very time-consuming in the presence of redundant constraints. This is because if near the end of the process $B$ is left with redundant constraints only, no solution of $C_L$ can violate any of these constraints, simply because these constraints, being implied, will be surely satisfied.

Given the importance of query generation in the acquisition process, it is of primary importance that it is executed as efficiently as possible, and in a way such that the problem of premature convergence is avoided as much as possible. Towards this, we propose to exploit partial queries at this step of the process. Both QuAcq and MultiAcq, and also our proposed algorithm, assume that the user, be it human or machine, is able to answer partial queries, so there is no reason to limit the use of partial queries to the case where a complete query has been classified as negative. 

Our proposal is to model the query generation problem as an optimization problem, like in the previous section, in which we seek to find a (partial) assignment of the variables that maximizes the number of violated constraints in $B$. That is, we again have 
a MAX-CSP, with the difference being that the optimal solution does not necessarily involve an assignment to all the variables.  This optimization problem can be formally stated as ``search for $(e_Y$, $Y)$ with $e_Y \in sol(C_L[Y]) \land Y \subseteq X$, maximizing $\kappa_B(e_Y)$''. We call this heuristic {\em max$_B$}.  This is related to but is not the same as the {\em max} heuristic (described in the previous section) that was also used within QuAcq \cite{bessiere2013constraint}. As already explained, the {\em max} heuristic tries to generate a complete solution of $C_L$ that violates a maximum number of constraints from $B$. Hence, given a time limit, which is necessary for any algorithm to run in reasonable times as explained above, {\em max} will focus on finding complete assignments that satisfy all the constraints in $C_L$ and violate as many as possible from $B$, while {\em max$_B$} will focus on violating as many constraints as possible from $B$ without necessarily building a complete variable assignment.

Of course, finding a partial assignment $e_Y$ on a set of variables $Y \subseteq X$ not rejected by $C_L$ and violating constraints from the bias, does not mean that $e_Y$ can be extended to a solution of $C_L$, but this is not a problem under the assumption that the user can classify partial queries correctly.

Although the difference between  {\em max$_B$} and {\em max} may not seem substantial, experimental results given below show that the use of {\em max$_B$} largely alleviates the danger of premature convergence and can have a significant impact on the total run time of the acquisition algorithm. This is because by using {\em max$_B$}, the system can also learn redundant constraints, thus it does not have to prove that a constraint cannot be violated. Learning the redundant constraints is necessary if we want to guarantee that the system will always converge.

\subsection{Variable ordering heuristic}
\label{sec:var}

Given that query generation is an optimization problem that is solved using a CP solver, an important question that must be answered is which variable/value ordering heuristic to use for this problem. One way is to simply apply the default strategies of the CP solver used. For example, dom/wdeg \cite{boussemart2004boosting} or a simpler heuristic like dom/ddeg (or even dom) 
for variable ordering, and random or lexicographic value ordering. This is the path taken by all acquisition algorithms so far. The reasoning behind it is that standard heuristics will help find a complete solution to $C_L$ quite fast, as this problem is a typical CSP, and then the maximum number of violating constraints in $B$ will be seeked within the time limit as a secondary requirement.

But, as explained before, in  constraint acquisition we would ideally want every generated query to contain as much ``information'' as possible.
Given that the focus of our proposed heuristic {\em max$_B$}, is on violating as many constraints from $B$ as possible, and not on building a complete solution to $C_L$, it is likely that traditional variable and value ordering heuristics, that are efficient when seeking a solution to a CSP, are not the best choice. 
This is because these heuristics focus on information (conflicts, degrees, domain sizes) obtained from the variables and constraints of $C_L$. But since {\em max$_B$} primarily focuses on the bias $B$, meaning that finding a complete solution to $C_L$ is not paramount, it is reasonable to use a variable ordering heuristic that exploits information obtained from the variables and constraints of $B$. Towards this, we propose a simple variable ordering heuristic that maximizes the constraint violations in $B$.  

This heuristic, which we call {\em bdeg} (degree of variables in the bias $B$), selects the variable which participates in the maximum number of constraints present in $B \setminus C_L$. It can be seen as analogue to the classic variable ordering heuristic {\em deg} for CSPs, which selects the variable with the largest degree. But in contrast to standard CSPs where deg is not competitive at all, bdeg is very efficient when used for query generation, especially near convergence where it manages to drastically cut down the waiting times for the user. 

Comparing bdeg to heuristics like dom/wdeg, we note that bdeg prefers variables belonging to the scope of many constraints from $B$, ignoring $C_L$, whereas standard heuristics do not treat the constraints from $C_L$ and $B$ differently when computing their metric, or even focus only on constraints from $C_L$. Hence, for example, a variable involved in many constraint from $C_L$ and only few from $B$ is very likely to be preferred. 

\subsection{Value Ordering Heuristic}
\label{sec:value}

First recall that a generated query should violate a large number of constraints from $B$ and some of these constraints may be added to $C_L$, while others may be removed from $B$, depending on the answers of the user to the partial queries posted by {\em FindAllCons}. 
So, using a value ordering heuristic that picks values that are involved in a large number of conflicts (i.e. violate many constraints) in $B$ makes sense for the query generation step.  This is because picking a value which violates many constraints from $B$ quickly leads to a (partial) assignment with a lot of information to be extracted.

Hence, we propose the {\em max$_v$} heuristic for value ordering, which selects the value that maximizes the number of conflicts (constraint violations) between the currently selected variable and the variables that have been already assigned.
We consider instantiated variables only because if {\em max$_B$} is used then the query generated may not include all the variables of the problem. 

To better understand the practical effect of the {\em max$_v$} heuristic, let us consider its behaviour when generating a query in the problem from our running example.

\begin{example}

For the purposes of this example, the variable ordering heuristic is not important. Assume that variables are ordered lexicographically. In the beginning of the query generation process, the use of {\em max$_v$} will not have any effect at the first variable's instantiation. Assume that the value selected for the first variable is $x_1 = 1$. From now on, when choosing a value for the following variables, {\em max$_v$} will keep on choosing value 1, as in this example it is the only value that violates some constraints from $B$. Thus, the generated example will be $e = \{1,1,1,1,1,1,1,1\}$, which in this case violates all the constraints from the bias.

\end{example}

An important factor to consider about the effect of value ordering regards the generation of the first query. 
In this case all the variables are involved in the same number of constraints in $B$ (because nothing has been removed from $B$ yet) and $C_L$ is empty. Thus, the variable ordering is not important. On the other hand, value ordering seems to be very important.
Experiments have shown that using {\em max$_v$} for the generation of the first query often leads to a query which violates {\em all} of the constraints from $C_T$. Also, when the acquisition process is near convergence and only a few constraints are left in $B$, we want queries that remove a maximum number of constraints after a positive answer from the user. So, again the {\em max$_v$} heuristic is the best choice.

Given the importance of selecting values that maximize constraint violations, the lexico value ordering heuristic, which simply selects values in their lexicographic order and is very commonly used by CP solvers, 
is not a good idea because the values that maximize conflicts may appear in the middle or near the end of a domain, meaning that values with low conflicts may be preferred instead. On the other hand, random value ordering is a better idea but still not as good as focusing on the values that maximize conflicts, as {\em max$_v$} does. As a downside, {\em max$_v$} is more expensive to compute as the conflicts caused by each value must be calculated before the selection is made.

\section{Experimental Evaluation}
\label{sec:experiments}

To evaluate our proposed methods, we ran experiments on a system carrying an Intel(R) Xeon(R) E5-2667 CPU, 2.9 GHz,
with 8 Gb of RAM. We compared the proposed methods to both QuAcq and MultiAcq, which were implemented as efficiently as possible using the strategies described in \cite{bessiere2013constraint,arcangioli2016multiple}.

To be precise:

\begin{itemize}

\item
The ``flawed'' {\em FindC} function of \cite{bessiere2013constraint} described in Section~\ref{sec:quacq} is used in all our methods. This does not affect our results as a situation in which it may fail (analyzed in Section~\ref{sec:quacq}) 
does not appear in any of the studied benchmarks. However, we have also implemented the corrected version of \cite{bessiere2016new} to deal with such cases when they arise.

\item
In all our methods, and also in QuAcq and MultiAcq, we set the $cut_{max}$ cutoffs of the query generation step (described in Section~\ref{sec:query}) to 5 seconds. This means that if no query is found within 5 seconds, Function {\em QGen} returns. Also, we set the $cut_{min}$ cutoff to 1 second, returning the best example found within this time limit, if any is found.

\item
To maximize the performance of MultiAcq we used the heuristic proposed in \cite{arcangioli2016multiple}: 
A cutoff of 5 seconds is used in function {\em FindAllScopes}. After triggering the cutoff for the first time, {\em FindAllScopes} is called again on the same complete example with a reverse order of the variables. If the cutoff is triggered for a second time, we generate a new example and shuffle the variables' order. To ensure termination, {\em FindAllScopes} is forced to return at least one scope before cutting off. 

\item
To evaluate our proposed variable and value ordering heuristics, we implemented them within QuAcq and MQuAcq and ran experiments using benchmark instances. We compared them with the use of {\em dom/wdeg} as variable ordering heuristic and {\em random} for value ordering, which are the standard options for existing acquisition algorithms.

\item In order to compare all the algorithms on the same scenario, all the experiments concern the extreme case where no background knowledge is used and thus $C_L$ is initially empty. This extreme scenario results in an overall number of queries that may seem too large for human users to answer without making errors. However, in real applications background knowledge can be used either by giving a frame of basic constraints to the system or by using some other method to extract some constraints from known solutions and non-solutions of the problem, e.g. ModelSeeker~\cite{beldiceanu2012model}.

\item
Each method was run 10 times and the means are presented.

\end{itemize}

We used the following benchmarks in our study:

\textbf{Sudoku}. The Sudoku puzzle is a 9 $\times$ 9 grid. It must be completed in such a way that all the rows, all the columns and the 9 non-overlapping 3 $\times$ 3 squares contain the numbers 1 to 9. The {\em vocabulary} for this problem has 81 variables and domains of size 9. The target network has 810 binary $\neq$ constraints on rows, columns and squares. The bias was initialized with 12.960 binary constraints from the language $ \Gamma = \{=, \neq, >, < \}$.

\textbf{Greater than Sudoku} (GTSudoku). This is a variant of Sudoku where instead of having only cliques of $\neq$ constraints, some neighboring variables are related via $>, <$ constraints. The number of variables and the maximum domain size are the same as in Sudoku, but there are no prefilled squares (i.e. assigned variables). We used the instance shown in Figure~\ref{fig:gtsudoku} and the same language and bias as in Sudoku. What is interesting with GTSudoku is that the introduction of inequality constraints breaks up the regular structure of Sudoku. 

\begin{figure}[h]

\centerline{\includegraphics[width=2.2in]{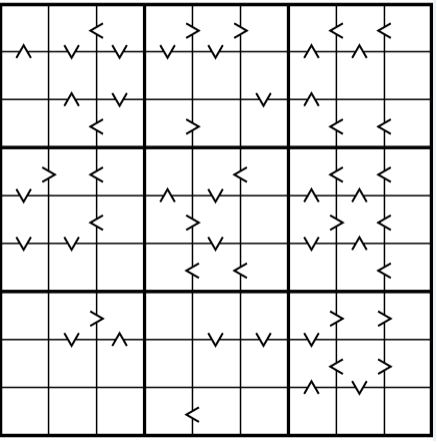}}
\caption{Greater than Sudoku instance used in the experiments}

\label{fig:gtsudoku}

\end{figure}

\textbf{Latin Square}. The Latin square problem consists of a n $\times$ n table in which each element occurs once in every row and column. In our experiments we set n to 10, meaning that we have 100 variables with domains of size 10. The target network has 900 binary $\neq$ constraints on rows and columns. The system was initialized with a bias of 19.800 binary constraints created from the language $\Gamma = \{=, \neq, >, < \}$.

\textbf{Zebra}. The Zebra problem has a single solution. The problem 
consists of 25 variables of domain size of 5. The target network contains 50 $\neq$ constraints and 12 additional constraints given in the description of the problem. The bias was initialized with 1200 binary constraints from the language $\Gamma = \{=, \neq, >, <, x_i - x_j = 1, |x_i - x_j| = 1\}$.

\textbf{Murder}. Someone was murdered and there are 5 suspects, each one having an item, an activity and a motive for the crime. The problem is to find the murderer. 
This problem consists of 20 variables (the 5 suspects and their items, activities and motives) 
with domains of size 5. The target network contains 4 cliques of $\neq$ constraints and 12 additional binary constraints, given as clues in the description of the problem.  The bias was initialized with 760 constraints based on the language $\Gamma = \{=, \neq, >, < \}$.

\textbf{Purdey’s general store}~\cite{purdey}. Four families stopped by Purdey’s general store, each one aiming to buy a different item and paying in a different way. Under a set of additional constraints given in the description of the problem, the goal is to match each family with the item they bought and how they paid for it. It has a single solution. It is modelled with 12 variables (for families, items and paying methods), with domains of size 4. The target network consists of 27 constraints. The bias was initialized with 264 constraints based on the language $\Gamma = \{=, \neq, >, < \}$.

\textbf{Allergy}. A problem crafted by the XCSP team. There are 4 people having allergies, and 8 products are given in the description of the problem. Based on the constraints given, the goal is to find who has an allergy on which product. It consists of 12 variables with domains of size 4 and 26 constraints. The bias was initialized with 264 constraints based on the language $\Gamma = \{=, \neq, >, < \}$.

\textbf{Golomb rulers}. The problem is to find a ruler where the distance between any two marks is different from that between any other two marks. We built a simplified version of a Golomb ruler with 12 marks, with the target network consisting only of quaternary constraints\footnote{The ternary constraints derived when $i = k$ or $j = l$ in $|x_i - x_j| \neq |x_k - x_l|$ were excluded from the target network}. In total $C_T$ consists of 495 constraints.
The bias was created with the language $\Gamma = \{=, \neq, >, <, |x_i - x_j| = |x_k - x_l|, |x_i - x_j| \neq |x_k - x_l| \}$, including binary and quaternary constraints. In total $B$ contained 1254 constraints.

\textbf{Exam Timetabling Problem} (Exam TT). We used a simplified version of the exam timetabling problem of the Department of Electrical and Computer Engineering of the University of Western Macedonia, Greece. We considered 24 courses and 2 weeks of exams, meaning that there are 10 possible days for each course to be assigned. We assumed that there are 3 timeslots available in each day. This resulted in a model with 24 variables and domains of size 30. There are $\neq$ constraints between any two courses, assuming that only one course is examined during each time slot, i.e. there is only one available room for examination. Also, there exist constraints prohibiting courses of the same semester being examined on the same day. As we assumed that there are 3 timeslots in each day, the constraint preventing 2 courses of the same semester being in the same day was modeled as $|\lfloor x_i/3 \rfloor - \lfloor x_j/3 \rfloor| > 0$. Hence, the language use was $\Gamma = \{=, \neq, >, <, |\lfloor x_i/3 \rfloor - \lfloor x_j/3 \rfloor| > y\}$, with 5 different values for $y$. This resulted in a bias of 3864 constraints.

\textbf{Radio Link Frequency Assignment Problem}.
The RLFAP is the problem of providing communication channels from limited spectral resources~\cite{cabon1999radio}. We use a simplified version of the problem, which consists of 50 variables with domains of size 40. The target network contains 125 binary distance constraints. We built the bias using a language of 2 basic distance constraints ($\{|x_i - x_j| > y, |x_i - x_j| = y\}$) with 5 different possible values for $y$. This led to a language of 10 different distance constraints.
In total, $B$ contained 12250 constraints.

In our experiments we measure the size of the learned network $C_L$, the total number of queries $\#q$, the average size $\bar{q}$ of all queries, the number of complete queries $\#q_c$, the average waiting time $\bar{T}$ (in secs) for the user, the maximum waiting time $T_{max}$ (in secs) for the user, the time $T_{queries}$ taken from the start of the process until the last query and the total time needed (to converge) $T_{total}$. The difference between $T_{total}$ and $T_{queries}$ is the time needed to prove convergence or to reach premature convergence (because of the cutoffs). The size of $C_L$ in some cases is smaller than the size of the target network $C_T$ due to the presence of redundant constraints that some methods learn and others do not. In addition, we counted the times each method triggers any of the two cutoffs.  

We first demonstrate the performance of MQuAcq and {\em FindScope-2} on these benchmarks, compared to the existing methods (Section~\ref{sec:mquacq-eval}). Then in Section~\ref{sec:heur-eval} we evaluate the proposed heuristics. In Section~\ref{sec:bias} we evaluate the effect of the size of the bias on the performance of MQuAcq. Finally, in Section~\ref{sec:scaling} we investigate our algorithm's scalability.

\subsection{MQuAcq and {\em FindScope-2} evaluation}
\label{sec:mquacq-eval}

For the experiments presented here all the methods compared, including QuAcq and MultiAcq, use the {\em max} heuristic (described in Section~\ref{sec:query}) for the query generation step, with {\em dom/wdeg} for variable ordering and {\em random} value ordering. In Table \ref{res:comp} we evaluate our proposed algorithm MQuAcq and the new function {\em FindScope-2} and we compare them against the existing methods. Hence, we give results from QuAcq, MultiAcq, MQuAcq, QuAcq with {\em FindScope-2} instead of {\em FindScope} and MQuAcq with {\em FindScope-2}. 

We do not present results from the RLFAP benchmark for MultiAcq, as it did not manage to converge after running for 24 hours. This was due not only to its linear complexity in terms of the number of queries, but also because the bias contains many constraints in each possible scope and thus the condition at line 2 of {\em FindAllScopes} does not help to avoid redundant searches. 
The other algorithms also suffer from high cpu times, but only when they are trying to generate queries near convergence. The large number of constraints in each scope, and particularly the existence of constraints that are implied by other constraints (e.g. $\{|x_i - x_j| > y_1$ implies $\{|x_i - x_j| > y_2$ if $y_1 > y_2$) can cause the appearance of a large number of constraints in the bias that cannot be violated, resulting in high convergence times. A similar problem is present in GTSudoku, again because of implied constraints that appear in the bias.
As we explain in Section~\ref{sec:heur-eval}, our proposed heuristics from Section~\ref{sec:heur} can alleviate this problem.

\begin{table}[htbp]
\begin{footnotesize}
\centering
\caption{Results of MQuAcq and {\em FindScope-2}}
{
\resizebox{\textwidth}{!}{%
\begin{tabular}{ |l|l|r|r|r|r|r|r|r|r|  }
\hline
Benchmark & Algorithm & $|C_L|$ & $\#q$ & $\bar{q}$ & $\#q_c$ & $\bar{T}$ & $T_{max}$ & $T_{queries}$ & $T_{total}$ \\
\hline
 & QuAcq & 648 & 11529 & 35 & 659 & 0.061 & 1.14 & 708.76 & 1529.78 \\ 
 & MultiAcq & 796 & 14508 & 10 & 39 & 0.071 & 36.76 & 1034.52 & 1119.69 \\ 
 & MQuAcq & 803 & 14935 & 26 & 37 & 0.010 & 20.49 & 154.47 & 194.57 \\ 
 & QuAcq + {\em FindScope-2} & 648 & 5960 & 43 & 659 & 0.119 & 1.15 & 710.57 & 1531.58 \\ 
\multirow{-5}{*}{Sudoku} & MQuAcq + {\em FindScope-2} & 801 & 6865 & 32 & 40 & 0.026 & 15.33 & 175.14 & 225.15 \\
\hline
 & QuAcq & 634 & 11325 & 35 & 649 & 0.82 & 1140.45 & 9235.54 & 11217.63 \\
 & MultiAcq & 747 & 16324 & 15 & 70 & 0.78 & 1383.72 & 12522.16 & 13917.67 \\
 & MQuAcq & 732 & 13912 & 26 & 45 & 0.40 & 905.68 & 5564.73 & 6959.69 \\
 & QuAcq + {\em FindScope-2} & 636 & 5950 & 42 & 653 & 1.51 & 1582.52 & 8987.09 & 10920.54 \\
\multirow{-5}{*}{GTSudoku} & MQuAcq + {\em FindScope-2} & 742 & 6663 & 31 & 52 & 0.86 & 970.77 & 5720.09 & 7003.27 \\
\hline
 & QuAcq & 855 & 15489 & 46 & 870 & 0.066 & 10.17 & 1020.83 & 1251.22 \\ 
 & MultiAcq & 899 & 21079 & 11 & 52 & 0.163 & 20.27 & 3429.16 & 3439.18 \\ 
 & MQuAcq & 899 & 17842 & 37 & 49 & 0.010 & 5.23 & 171.75 & 181.77 \\ 
 & QuAcq + {\em FindScope-2} & 855 & 8115 & 55 & 873 & 0.127 & 10.15 & 1028.85 & 1259.23 \\ 
\multirow{-5}{*}{Latin} & MQuAcq + {\em FindScope-2} & 899 & 8228 & 46 & 50 & 0.023 & 10.30 & 189.34 & 199.38 \\ 
\hline
 & QuAcq & 60 & 775 & 11 & 60 & 0.069 & 1.03 & 53.68 & 53.69 \\ 
 & MultiAcq & 57 & 975 & 6 & 8 & 0.264 & 127.62 & 257.77 & 257.78 \\ 
 & MQuAcq & 59 & 783 & 8 & 8 & 0.006 & 1.03 & 4.37 & 4.37 \\ 
 & QuAcq + {\em FindScope-2} & 60 & 496 & 12 & 60 & 0.109 & 1.03 & 54.08 & 54.03 \\ 
\multirow{-5}{*}{Zebra} & MQuAcq + {\em FindScope-2} & 59 & 469 & 10 & 7 & 0.009 & 1.03 & 4.08 & 4.09 \\ 
\hline
 & QuAcq & 52 & 599 & 9 & 52 & 0.085 & 1.01 & 51.16 & 51.48 \\ 
 & MultiAcq & 52 & 704 & 5 & 8 & 0.025 & 3.65 & 17.34 & 17.65 \\ 
 & MQuAcq & 52 & 619 & 6 & 8 & 0.012 & 1.01 & 7.28 & 7.75 \\ 
 & QuAcq + {\em FindScope-2} & 52 & 356 & 10 & 52 & 0.144 & 1.01 & 51.31 & 51.63 \\ 
\multirow{-5}{*}{Murder} & MQuAcq + {\em FindScope-2} & 52 & 374 & 8 & 7 & 0.028 & 1.01 & 6.56 & 6.89 \\ 
\hline
 & QuAcq & 26 & 282 & 5 & 26 & 0.07 & 1.00 & 19.15 & 19.16 \\
 & MultiAcq & 27 & 234 & 4 & 5 & 0.01 & 1.01 & 2.21 & 2.22 \\
 & MQuAcq & 26 & 269 & 4 & 5 & 0.01 & 1.00 & 2.06 & 2.08 \\
 & QuAcq + {\em FindScope-2} & 26 & 170 & 6 & 26 & 0.11 & 1.00 & 19.23 & 19.25 \\
\multirow{-5}{*}{Purdey} & MQuAcq + {\em FindScope-2} & 26 & 149 & 5 & 5 & 0.01 & 1.00 & 2.13 & 2.14 \\
\hline
 & QuAcq & 26 & 283 & 5 & 26 & 0.06 & 1.00 & 17.83 & 17.85 \\
 & MultiAcq & 26 & 226 & 4 & 4.8 & 0.01 & 1.00 & 2.16 & 2.18 \\
 & MQuAcq & 26 & 267 & 4 & 5 & 0.01 & 1.00 & 1.94 & 2.06 \\
 & QuAcq + {\em FindScope-2} & 26 & 169 & 6 & 26 & 0.11 & 1.00 & 17.87 & 17.88 \\
\multirow{-5}{*}{Allergy} & MQuAcq + {\em FindScope-2} & 26 & 151 & 4 & 5 & 0.01 & 1.00 & 2.11 & 2.11 \\
\hline
 & QuAcq & 495 & 7585 & 6 & 496 & 0.069 & 1.17 & 525.98 & 526.08 \\ 
 & MultiAcq & 495 & 2368 & 6 & 64 & 0.029 & 1.18 & 68.82 & 68.92 \\ 
 & MQuAcq & 495 & 6350 & 5 & 72 & 0.012 & 1.18 & 78.76 & 78.86 \\ 
 & QuAcq + {\em FindScope-2} & 495 & 1552 & 9 & 496 & 0.338 & 1.18 & 524.97 & 525.07 \\ 
\multirow{-5}{*}{Golomb-12} & MQuAcq + {\em FindScope-2} & 495 & 961 & 8 & 69 & 0.082 & 3.87 & 79.10 & 79.20 \\ 
\hline
 & QuAcq & 276 & 3856 & 11 & 277 & 0.07 & 1.14 & 281.11 & 576.17 \\
 & MultiAcq & 276 & 3086 & 7 & 35 & 0.73 & 341.76 & 2264.81 & 2553.40 \\
 & MQuAcq & 276 & 3747 & 9 & 36 & 0.08 & 126.13 & 311.70 & 592.94 \\
 & QuAcq + {\em FindScope-2} & 276 & 1451 & 14 & 277 & 0.19 & 1.04 & 281.66 & 576.86 \\
\multirow{-5}{*}{Exam TT} & MQuAcq + {\em FindScope-2} & 276 & 1222 & 11 & 36 & 0.24 & 100.57 & 296.69 & 584.93 \\
\hline
 & QuAcq & 102 & 1705 & 26 & 166 & 3.657 & 890.60 & 6,235.19 & 7,513.17 \\ 
 & MultiAcq & - & - & - & - & - & - & - & - \\
 & MQuAcq & 122 & 2492 & 24 & 107 & 2.067 & 933.00 & 5,150.23 & 6,308.14 \\ 
 & QuAcq + {\em FindScope-2} & 102 & 1096 & 29 & 167 & 6.163 & 896.26 & 6,755.14 & 7,629.21 \\ 
\multirow{-5}{*}{RLFAP} & MQuAcq + {\em FindScope-2} & 122 & 1442 & 25 & 107 & 3.380 & 932.00 & 4,873.96 & 6,204.14 \\ 
\hline
\end{tabular}}
}
\label{res:comp}
\end{footnotesize}
\end{table}

Looking at the performance of MQuAcq, and comparing it to QuAcq, we observe that the use of {\em FindAllCons} to learn all the violated constraints from a negative example reduces significantly the average waiting time per query for the user and the total time of the execution in all benchmarks except RLFAP and GTSudoku (due to the nature of the problems, as mentioned above), where the time needed is still reduced but only by a little. Also, in Exam TT, QuAcq and MQuAcq have similar performance in terms of average time and total time. This is because although MQuAcq learns faster most of the constraints, it needs a lot more time for the generation of the last queries, due to the structure of the problem, as several constraints from the target network not learned yet (i.e. they are still in the bias) are difficult to be violated when the learned network is satisfied. This is confirmed by considering the maximum time that the user had to wait for a query to be generated and posted.

Regarding the rest of the problems, QuAcq is 8 times slower than MQuAcq in Sudoku and Allergy, 7 times in Latin square, 12 times in Zebra, 9 times in Purdey, and 6.5 times in Murder and Golomb rulers. This is due to the fewer generations of new examples in line 5 of MQuAcq, because the algorithm is able to learn a maximum number of violated constraints from each negative example. This is validated by looking at column $\#q_c$, which shows that far fewer complete queries are generated. As a downside, MQuAcq requires more queries in total than QuAcq to converge in most cases, and the difference is more evident on Sudoku, GTSudoku and Latin. However, as we can see on these problems MQuAcq learns a greater number of constraints of the target network than QuAcq, and the average size of the queries posted by MQuAcq is smaller. Also, we can observe that in Golomb rulers, which contains quaternary constraints, the queries posted to the user by MQuAcq were fewer.

Comparing MQuAcq to MultiAcq, it is clear that the redundant searches made by MultiAcq greatly affect the average time per query and total time needed for the system to converge. MQuAcq needs far less time to ask a query to the user, and requires posting fewer queries to converge, on most problems. On the other hand, on Golomb Rulers, MultiAcq displays better performance both in number of queries and in total time. This can be explained as {\em FindScope}, that is used by MQuAcq, posts a lot of redundant queries to the user and also the problem consists of only 12 variables, so the branching of MultiAcq is not very time-consuming.

Focusing on {\em FindScope-2} when used inside QuAcq, we can see that the number of queries posted to the user 
were significantly lower compared to standard QuAcq with {\em FindScope}, because the former avoids asking several redundant queries. In terms of the number of queries, {\em FindScope-2} gives a gain of $35\%$ on the RLFAP problem, $36\%$ on the Zebra problem, $40\%$ on Murder, Purdey and Allergy, $48\%$ on Sudoku, GTSudoku and Latin square, $62\%$ on Exam TT and up to $80\%$ on Golomb Rulers. 
Interestingly, it seems that the more variables are present in a problem, the bigger is the gain in avoided queries. As a downside, {\em FindScope-2} increases the average waiting time between the queries, but not the total time required to converge. The average time is increased simply because some queries are not posted because they would be redundant. In addition, as we can observe from the results from Golomb Rulers, the reduction in the number of queries in problems with higher arity constraints is even bigger.

The results obtained from MQuAcq with {\em FindScope-2} show that the use of {\em FindScope-2} has the same effect on MQuAcq as on QuAcq, cutting down the number of queries significantly, from $40\%$ (in Murder) up to $85\%$ (in Golomb). Comparing to MultiAcq, now the number of queries posted to the user is considerable lower, from $33\%$ (in Allergy) up to $61\%$ (in Latin squares).

Regarding the cutoffs, 
neither of the two cutoffs  was triggered by any method on Zebra, Murder, Purdey, Allergy and Golomb. On Sudoku, QuAcq (resp. MultiAcq) triggered the first cutoff 2 (resp. 3) times on average and the second 170 (resp. 26) times. On Latin squares these numbers were 9 and 46 for QuAcq and 16 and 19 for MultiAcq. MQuAcq triggered the first cutoff 5 times on average on Sudoku and the second also 5 times. On Latin square these numbers were 11 and 5 respectively. Given that the triggering of the cutoffs is associated with the problem of premature convergence, as we explain at the end of Section~\ref{sec:quacq}, the lower numbers for MQuAcq indicate that it is less likely to terminate with premature convergence. On the other hand, on Exam TT, QuAcq triggered the first cutoff only once, while MQuAcq and MultiAcq triggered it 5 times. The second cutoff was triggered 57 times from QuAcq, 98 times from MultiAcq and 104 from MQuAcq.
On RLFAP and GTSudoku, the cutoffs were triggered too many times due to the reasons explained above. On average, QuAcq triggered the first cutoff 21 times and the second 1444 times on RLFAP (resp. 98 and 1757 on GTSudoku). The corresponding numbers for MQuAcq were 20 and 1195 on RLFAP (resp. 17 and 1356 on GTSudoku). MultiAcq triggered the first cutoff 27 times and the second 2598 times on GTSudoku.

In the remainder of the experimental evaluation we will compare our methods only against QuAcq, as it is clear that learning a maximum number of constraints from each generated query using MQuAcq is more efficient than with MultiAcq. Also, we will present the results of both QuAcq and MQuAcq with the use {\em FindScope-2} instead of {\em FindScope}.

\subsection{Evaluation of heuristics}
\label{sec:heur-eval}

In this section we first evaluate the heuristic {\em max$_B$} for the query generation step in tandem with {\em bdeg} for variable ordering. Next, we focus on the performance of the proposed value ordering heuristic. 

\subsubsection{{\em max$_B$} for the query generation step}
\label{sec:maxb-eval}

Recall that the objective of the {\em max$_B$} heuristic is to find a (partial) assignment that maximizes the number of violated constraints from $B$ instead of focusing on finding a complete solution of $C_L$ as {\em max} does. Hence, if {\em max$_B$} is used to generate queries, the variable ordering heuristic should comply with this objective. Our intuition behind the proposed variable ordering heuristic {\em bdeg} is that standard heuristics like {\em dom/wdeg} are not suitable for use in conjunction with {\em max$_B$}.
On the other hand, such heuristics are better suited to be used in tandem with {\em max} whose objective is to find a complete solution of $C_L$ quickly. 

We use Sudoku as a sample problem to confirm the above assumptions. In Figures~\ref{fig:max-maxb-quacq} and~\ref{fig:max-maxb-mquacq} we report the cpu time performance of the QuAcq and MQuAcq algorithms using {\em max} and {\em max$_B$} with {\em bdeg} and {\em dom/wdeg}. In all cases we use random value ordering. Specifically, the figures depict the cpu time required by each combination of heuristics to learn an increasing portion of the target network (the x-axis gives the number of constraints learned). 

The results confirm our intuition. When {\em max} is used for query generation within QuAcq (Figure~\ref{fig:max-maxb-quacq}a), the algorithm is by far faster with {\em dom/wdeg} compared to {\em bdeg}. In contrast, when  {\em max$_B$} is used (Figure~\ref{fig:max-maxb-quacq}b) then the choice of variable ordering heuristic does not affect the run time initially, but as convergence is approached, {\em bdeg} speeds up the process considerably because {\em dom/wdeg} takes too long to generate the last few queries compared to {\em bdeg} which finds partial assignments really fast. Considering MQuAcq, when {\em max} is used (Figure~\ref{fig:max-maxb-mquacq}a), {\em bdeg} is slightly faster initially, but is outperformed by {\em dom/wdeg} near convergence. On the other hand, when {\em max$_B$} is used  (Figure~\ref{fig:max-maxb-mquacq}b), {\em bdeg} and {\em dom/wdeg} are very close initially, but the former is again faster near convergence.




\begin{figure}[h]

\subfloat[]{\includegraphics[width=2.2in]{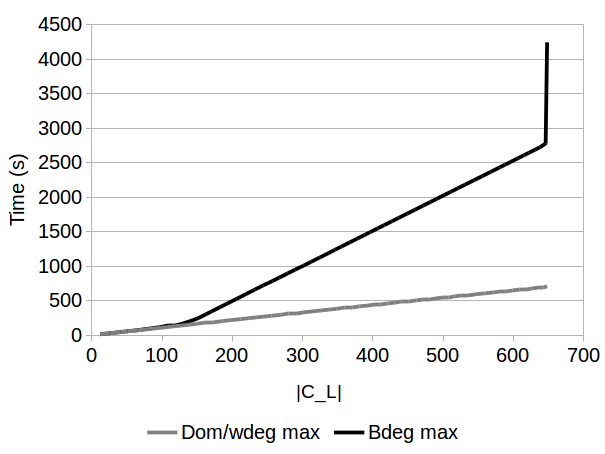}}
\qquad
\subfloat[]{\includegraphics[width=2.2in]{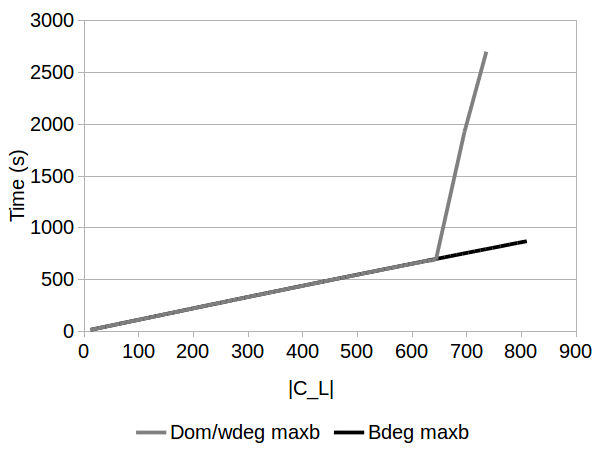}}

\caption{QuAcq using {\em max} and {\em max$_B$} with {\em bdeg} and {\em dom/wdeg} in the Sudoku problem}

\label{fig:max-maxb-quacq}

\end{figure}

\begin{figure}[h]

\subfloat[]{\includegraphics[width=2.2in]{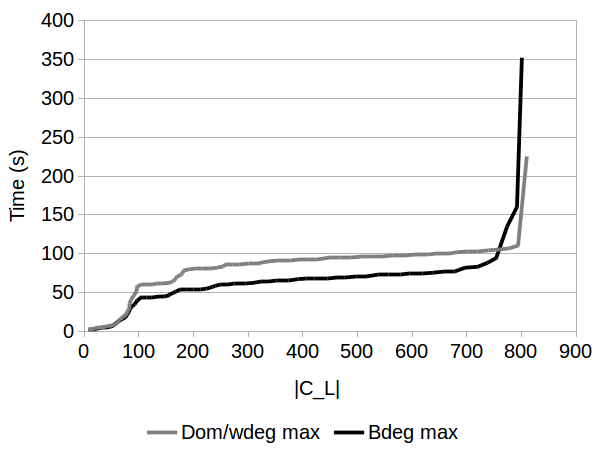}}
\qquad
\subfloat[]{\includegraphics[width=2.2in]{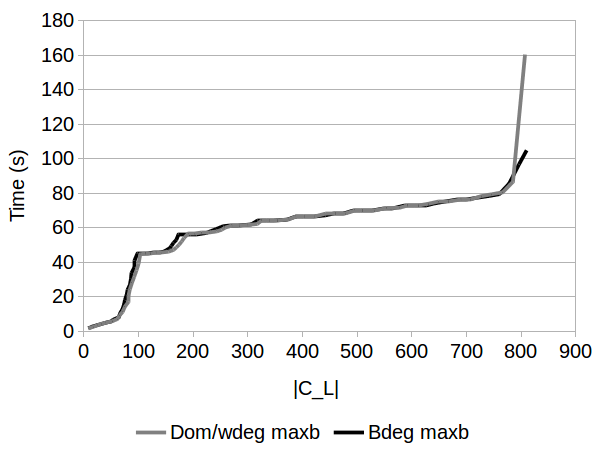}}

\caption{MQuAcq using {\em max} and {\em max$_B$} with {\em bdeg} and {\em dom/wdeg} in the Sudoku problem}

\label{fig:max-maxb-mquacq}

\end{figure}

For a closer look at the difference between {\em bdeg} and {\em dom/wdeg} near convergence, Figure~\ref{fig:heur} displays the number of constraints from $B$ that are violated (y-axis) during each of the last 20 generated queries (x-axis). This data was obtained by applying MQuAcq with {\em max$_B$} on the Sudoku benchmark. It is clear that {\em bdeg}, as a heuristic that orders the variables with information obtained from $B$, violates considerably more constraints than {\em dom/wdeg}. Hence, the queries generated using {\em bdeg} are more ``informative'', which explains its good performance near convergence when used in tandem with {\em max$_B$}. 

\begin{figure}[H]
\centerline{\includegraphics[width=3.5in]{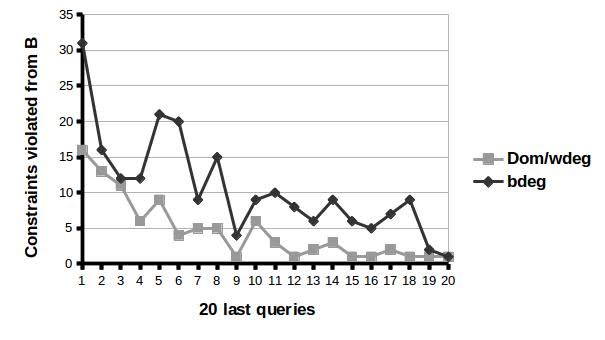}}
\caption{Number of constraints from the Bias that are violated when MQuAcq generates the last 20 queries in Sudoku.}
\label{fig:heur}
\end{figure}

To summarize, we have established that the use of {\em max$_B$} to generate queries requires the use of {\em bdeg} for variable ordering in order to maximize the performance of the acquisition algorithm, while if {\em max} is used to generate queries then {\em dom/wdeg} is a better option. We now compare these two strategies on all the considered benchmarks.

Table~\ref{res:maxb} displays the performance of 
{\em max$_B$} (with {\em bdeg}) when used inside QuAcq and MQuAcq compared to {\em max} (with {\em dom/wdeg}). We can see that on small problems (Zebra, Murder, Purdey, Allergy and Golomb) {\em max$_B$} has similar performance to {\em max}. This is because such problems have only a few variables, meaning that in most cases both {\em max$_B$} and {\em max} can find complete solutions to $C_L$ that violate many constraints in $B$ within the time limit.

\begin{table}[htbp]
\begin{footnotesize}
\centering
\caption{Comparing {\em max$_B$} to {\em max}.}
{
\resizebox{\textwidth}{!}{%
\begin{tabular}{ |l|l|r|r|r|r|r|r|r|r|  }
\hline
Benchmark & Algorithm & $|C_L|$ & $\#q$ & $\bar{q}$ & $\#q_c$ & $\bar{T}$ & $T_{max}$ & $T_{queries}$ & $T_{total}$ \\
\hline
 & QuAcq {\em max} & 648 & 5960 & 43 & 659 & 0.119 & 1.15 & 710.57 & 1531.58 \\
 & MQuAcq {\em max} & 801 & 6865 & 32 & 40 & 0.026 & 15.33 & 175.14 & 225.15 \\ 
 & QuAcq {\em max$_B$} & 810 & 6657 & 38 & 510 & 0.131 & 1.15 & 869.67 & 869.68 \\ 
\multirow{-4}{*}{Sudoku} & MQuAcq {\em max$_B$} & 810 & 6858 & 32 & 14 & 0.015 & 1.12 & 104.89 & 104.90 \\
\hline
 & QuAcq {\em max} & 636 & 5950 & 42 & 653 & 1.51 & 1582.52 & 8987.09 & 10920.54 \\
 & MQuAcq {\em max} & 742 & 6663 & 31 & 52 & 0.86 & 970.77 & 5720.09 & 7003.27 \\
 & QuAcq {\em max$_B$}  & 786 & 6493 & 24 & 278 & 0.19 & 5.37 & 1216.01 & 1219.98 \\
\multirow{-4}{*}{GTSudoku} & MQuAcq {\em max$_B$} & 787 & 6813 & 29 & 12 & 0.11 & 5.28 & 735.99 & 738.67 \\
\hline
 & QuAcq {\em max} & 855 & 8115 & 55 & 873 & 0.127 & 10.15 & 1028.85 & 1259.23 \\ 
 & MQuAcq {\em max} & 899 & 8228 & 46 & 50 & 0.023 & 10.30 & 189.34 & 199.38 \\ 
 & QuAcq {\em max$_B$} & 900 & 7946 & 54 & 793 & 0.126 & 1.19 & 999.17 & 999.18 \\ 
\multirow{-4}{*}{Latin} & MQuAcq {\em max$_B$} & 900 & 8411 & 46 & 17 & 0.017 & 1.16 & 142.85 & 142.86 \\ 
\hline
 & QuAcq {\em max} & 60 & 496 & 12 & 60 & 0.109 & 1.03 & 54.08 & 54.03 \\ 
 & MQuAcq {\em max} & 59 & 469 & 10 & 7 & 0.009 & 1.03 & 4.08 & 4.09 \\ 
 & QuAcq {\em max$_B$} & 60 & 481 & 12 & 56 & 0.110 & 1.03 & 52.84 & 52.84 \\ 
\multirow{-4}{*}{Zebra} & MQuAcq {\em max$_B$} & 60 & 480 & 10 & 6 & 0.009 & 1.03 & 4.53 & 4.54 \\ 
\hline
 & QuAcq {\em max} & 52 & 356 & 10 & 52 & 0.144 & 1.01 & 51.31 & 51.63 \\
 & MQuAcq {\em max} & 52 & 374 & 8 & 7 & 0.028 & 1.01 & 6.56 & 6.89 \\
 & QuAcq {\em max$_B$} & 52 & 370 & 10 & 47 & 0.136 & 1.01 & 50.23 & 50.44 \\ 
\multirow{-4}{*}{Murder} & MQuAcq {\em max$_B$} & 52 & 357 & 8 & 4 & 0.019 & 1.01 & 6.82 & 7.07 \\  
\hline
 & QuAcq {\em max} & 26 & 170 & 6 & 26 & 0.11 & 1.00 & 19.23 & 19.25 \\
 & MQuAcq {\em max} & 26 & 149 & 5 & 5 & 0.01 & 1.00 & 2.13 & 2.14 \\
 & QuAcq {\em max$_B$} & 26 & 171 & 6 & 23 & 0.12 & 1.00 & 20.13 & 20.13 \\
\multirow{-4}{*}{Purdey} & MQuAcq {\em max$_B$} & 26 & 152 & 5 & 3 & 0.02 & 1.00 & 2.29 & 2.30 \\
\hline
 & QuAcq {\em max} & 26 & 169 & 6 & 26 & 0.11 & 1.00 & 17.87 & 17.88 \\
 & MQuAcq {\em max} & 26 & 151 & 4 & 5 & 0.01 & 1.00 & 2.11 & 2.11 \\
 & QuAcq {\em max$_B$} & 26 & 169 & 6 & 23 & 0.12 & 1.00 & 19.46 & 19.47 \\
\multirow{-4}{*}{Allergy} & MQuAcq {\em max$_B$} & 26 & 150 & 4 & 3 & 0.01 & 1.00 & 2.13 & 2.13 \\
\hline
 & QuAcq {\em max} & 495 & 1552 & 9 & 496 & 0.338 & 1.18 & 524.97 & 525.07 \\ 
 & MQuAcq {\em max} & 495 & 961 & 8 & 69 & 0.082 & 3.87 & 79.10 & 79.20 \\ 
 & QuAcq {\em max$_B$} & 495 & 1789 & 9 & 438 & 0.294 & 1.19 & 526.77 & 526.89 \\ 
\multirow{-4}{*}{Golomb-12} & MQuAcq {\em max$_B$} & 495 & 970 & 8 & 49 & 0.086 & 1.18 & 83.00 & 83.12 \\ 
\hline
 & QuAcq {\em max} & 276 & 1451 & 14 & 277 & 0.19 & 1.04 & 281.66 & 576.86 \\
 & MQuAcq {\em max} & 276 & 1222 & 11 & 36 & 0.24 & 100.57 & 296.69 & 584.93 \\
 & QuAcq {\em max$_B$} & 276 & 1468 & 13 & 230 & 0.22 & 6.06 & 316.19 & 321.32 \\
\multirow{-4}{*}{Exam TT} & MQuAcq {\em max$_B$} & 276 & 1237 & 11 & 14 & 0.08 & 5.90 & 94.04 & 100.56 \\
\hline
 & QuAcq {\em max} & 102 & 1096 & 29 & 167 & 6.163 & 896.26 & 6,755.14 & 7,629.21 \\ 
 & MQuAcq {\em max} & 122 & 1442 & 25 & 107 & 3.380 & 932.00 & 4,873.96 & 6,204.14 \\ 
 & QuAcq {\em max$_B$} & 106 & 1094 & 26 & 77 & 0.242 & 6.63 & 264.62 & 268.32 \\ 
\multirow{-4}{*}{RLFAP} & MQuAcq {\em max$_B$} & 124 & 1445 & 24 & 25 & 0.115 & 6.55 & 165.76 & 173.90 \\ 
\hline
\end{tabular}}
}
\label{res:maxb}
\end{footnotesize}
\end{table}

On the other hand, on the bigger and harder problems (Sudoku, GTSudoku, Latin, Exam TT and RLFAP) the average and maximum time per query of both QuAcq and MQuAcq are all reduced when {\em max$_B$} is used, and so is the number of complete queries posted to the user. Also, the differences in the maximum time per query are quite large. 
Another observation is that for both QuAcq and MQuAcq, as column $|C_L|$ demonstrates, the use of {\em max$_B$} helps to not only learn the complete target network, but also redundant constraints. Inadvertently, this results in more queries being asked in some cases and greater $T_{queries}$ (e.g. QuAcq in Sudoku). On the other hand $T_{total}$ is significantly reduced. 
We can see that the use of {\em max$_B$} in QuAcq (resp. in MQuAcq) reduces the total time by $43\%$ ($53\%$) in Sudoku, $89\%$ ($89\%$) in GTSudoku, $20\%$ ($24\%$) in Latin, $44\%$ ($83\%$) in Exam TT and $96\%$ ($97\%$) in RLFAP.

These gains in average time per query and total cpu time can be explained because in Sudoku and Latin, any method that used {\em max$_B$} never triggered any cutoff, meaning that an irredundant query was always found in time. Accordingly, in GTSudoku, Exam TT and RLFAP, when {\em max} is used the cutoffs are triggered too many times, resulting in very high cpu times. On the other hand, when {\em max$_B$} is used, the second cutoff was never triggered in any of these problems while the first cutoff was triggered in average 40 times by QuAcq and 53 times by MQuAcq in GTSudoku only 7 times by both QuAcq and MQuAcq in EXAM TT and 15 times by QuAcq and 11 times by MQuAcq on RLFAP.

An issue that is not clearly visible from the data in the table is that of premature convergence. The difference between $T_{total}$ and $T_{queries}$ is in fact the time needed to reach (premature) convergence, because of the cutoffs. In general, the use of {\em max$_B$} alleviates the problem of premature convergence, as in all the benchmarks both the algorithms proved convergence because having learned the redundant constraints during the process, $B$ is empty in the end, and therefore the system does not have to prove that no solution of $C_L$ violates them.

\subsubsection{{\em max$_v$} for value ordering}
\label{sec:maxv-eval}

Now, let us focus on the use of the {\em max$_v$} value ordering heuristic. Table~\ref{res:maxv} illustrates the results obtained using random and {\em max$_v$} for value ordering alongside {\em bdeg} for variable ordering, in tandem with {\em max$_B$}.

\begin{table}[htbp]
\begin{footnotesize}
\centering
\caption{Comparing random value ordering to {\em max$_v$}.}
{
\resizebox{\textwidth}{!}{%
\begin{tabular}{ |l|l|r|r|r|r|r|r|r|r|  }
\hline
Benchmark & Algorithm & $|C_L|$ & $\#q$ & $\bar{q}$ & $\#q_c$ & $\bar{T}$ & $T_{max}$ & $T_{queries}$ & $T_{total}$ \\
\hline
 & QuAcq {\em rand} & 810 & 6657 & 38 & 510 & 0.131 & 1.15 & 869.67 & 869.68 \\ 
 & MQuAcq {\em rand} & 810 & 6858 & 32 & 14 & 0.015 & 1.12 & 104.89 & 104.90 \\
 & QuAcq {\em max$_v$} & 810 & 7074 & 37 & 555 & 0.123 & 1.14 & 868.22 & 868.23 \\ 
\multirow{-4}{*}{Sudoku} & MQuAcq {\em max$_v$} & 810 & 5101 & 4 & 3 & 0.215 & 1.30 & 1,095.18 & 1,095.20 \\ 
\hline
 & QuAcq {\em rand} & 786 & 6493 & 24 & 278 & 0.19 & 5.37 & 1216.01 & 1219.98 \\
 & MQuAcq {\em rand} & 787 & 6813 & 29 & 12 & 0.11 & 5.28 & 735.99 & 738.67 \\
 & QuAcq  {\em max$_v$} & 776 & 6598 & 23 & 259 & 0.18 & 5.37 & 1217.91 & 1223.17 \\
\multirow{-4}{*}{GTSudoku} & MQuAcq  {\em max$_v$} & 810 & 5144 & 4 & 2 & 0.29 & 5.18 & 1481.40 & 1486.71 \\
\hline
 & QuAcq {\em rand} & 900 & 7946 & 54 & 793 & 0.126 & 1.19 & 999.17 & 999.18 \\ 
 & MQuAcq {\em rand} & 900 & 8411 & 46 & 17 & 0.017 & 1.16 & 142.85 & 142.86 \\ 
 & QuAcq {\em max$_v$} & 900 & 8309 & 51 & 817 & 0.120 & 1.20 & 996.92 & 996.93 \\ 
\multirow{-4}{*}{Latin} & MQuAcq {\em max$_v$} & 900 & 6968 & 4 & 3 & 0.356 & 1.84 & 2,478.57 & 2,478.58 \\ 
\hline
 & QuAcq {\em rand} & 60 & 481 & 12 & 56 & 0.110 & 1.03 & 52.84 & 52.84 \\ 
 & MQuAcq {\em rand} & 60 & 480 & 10 & 6 & 0.009 & 1.03 & 4.53 & 4.54 \\ 
 & QuAcq {\em max$_v$} & 60 & 494 & 12 & 59 & 0.112 & 1.03 & 55.21 & 55.21 \\ 
\multirow{-4}{*}{Zebra} & MQuAcq {\em max$_v$} & 61 & 491 & 6 & 3 & 0.008 & 1.03 & 3.69 & 3.69 \\ 
\hline 
 & QuAcq {\em rand} & 52 & 370 & 10 & 47 & 0.136 & 1.01 & 50.23 & 50.44 \\ 
 & MQuAcq {\em rand} & 52 & 357 & 8 & 4 & 0.019 & 1.01 & 6.82 & 7.07 \\  
 & QuAcq {\em max$_v$} & 52 & 367 & 10 & 49 & 0.138 & 1.01 & 50.76 & 51.00 \\ 
\multirow{-4}{*}{Murder} & MQuAcq {\em max$_v$} & 52 & 365 & 4 & 3 & 0.007 & 1.01 & 2.67 & 2.67 \\ 
\hline
 & QuAcq {\em rand} & 26 & 171 & 6 & 23 & 0.12 & 1.00 & 20.13 & 20.13 \\
 & MQuAcq {\em rand} & 26 & 152 & 5 & 3 & 0.02 & 1.00 & 2.29 & 2.30 \\
 & QuAcq  {\em max$_v$} & 26 & 169 & 6 & 23 & 0.12 & 1.00 & 20.39 & 20.40 \\
\multirow{-4}{*}{Purdey} & MQuAcq  {\em max$_v$} & 27 & 146 & 3 & 2 & 0.01 & 1.00 & 1.66 & 1.66 \\
\hline
 & QuAcq {\em rand} & 26 & 169 & 6 & 23 & 0.12 & 1.00 & 19.46 & 19.47 \\
 & MQuAcq {\em rand} & 26 & 150 & 4 & 3 & 0.01 & 1.00 & 1.85 & 1.89 \\
 & QuAcq  {\em max$_v$} & 26 & 176 & 6 & 24 & 0.11 & 1.00 & 19.87 & 19.88 \\
\multirow{-4}{*}{Allergy} & MQuAcq  {\em max$_v$} & 26 & 149 & 3 & 3 & 0.01 & 1.00 & 1.72 & 1.72 \\
\hline
 & QuAcq {\em rand} & 495 & 1789 & 9 & 438 & 0.294 & 1.19 & 526.77 & 526.89 \\ 
 & MQuAcq {\em rand} & 495 & 970 & 8 & 49 & 0.086 & 1.18 & 83.00 & 83.12 \\ 
 & QuAcq {\em max$_v$} & 495 & 1740 & 9 & 473 & 0.306 & 3.26 & 532.96 & 533.06 \\ 
\multirow{-4}{*}{Golomb-12} & MQuAcq {\em max$_v$} & 495 & 567 & 3 & 2 & 0.109 & 1.25 & 61.78 & 61.87 \\ 
\hline
 & QuAcq {\em rand} & 276 & 1468 & 13 & 230 & 0.22 & 6.06 & 316.19 & 321.32 \\
 & MQuAcq {\em rand} & 276 & 1237 & 11 & 14 & 0.08 & 5.90 & 94.04 & 100.56 \\
 & QuAcq  {\em max$_v$} & 276 & 1808 & 13 & 277 & 0.16 & 2.99 & 285.23 & 293.84 \\
\multirow{-4}{*}{Exam TT} & MQuAcq  {\em max$_v$} & 276 & 1268 & 7 & 14 & 0.01 & 1.05 & 18.06 & 29.16 \\
\hline
 & QuAcq {\em rand} & 106 & 1094 & 26 & 77 & 0.242 & 6.63 & 264.62 & 268.32 \\ 
 & MQuAcq {\em rand} & 124 & 1445 & 24 & 25 & 0.115 & 6.55 & 165.76 & 173.90 \\ 
 & QuAcq {\em max$_v$} & 108 & 1073 & 23 & 91 & 0.152 & 8.00 & 163.33 & 170.04 \\ 
\multirow{-4}{*}{RLFAP} & MQuAcq {\em max$_v$} & 123 & 1377 & 9 & 1 & 0.042 & 7.63 & 57.12 & 70.43 \\ 
\hline
\end{tabular}}
}

\label{res:maxv}
\end{footnotesize}
\end{table}

Comparing against the results of random value ordering, we can see that the use of {\em max$_v$} does not affect the results of QuAcq significantly in terms of cpu time, but it has a negative effect on the number of queries required for the larger problems. This is because QuAcq does not use all the information included in each generated query, learning only one violated constraint.

With respect to MQuAcq, we observe that in RLFAP, Exam TT, Zebra and Murder, which include a small number of variables, the use of {\em max$_v$} reduces the total time and the average time per query (up to $70\%$ in Exam TT). In addition, in Golomb rulers the number of queries is reduced significantly ($42\%$) and the total time of the acquisition process is also reduced. In Purdey and Allergy, which are the smallest problems, there is no difference.
In contrast, in the other three benchmarks, which have a much larger $C_T$ (810 for Sudoku and GTSudoku and 900 for Latin), we observe that although the number of queries posted to the user is considerably reduced (by $25\%$ for Sudoku and GTSudoku and $17.2\%$ for Latin), 
the total time of the acquisition process and the average time per query are one order of magnitude higher compared to random value ordering. However, it can be seen that the maximum time the user has to wait for a query is not much higher. The total time and the average time per query are increased because of the branching that {\em FindAllCons} performs.
Another observation is that the use of {\em max$_v$} significantly reduces the average size per query as well as the number of complete queries posted to the user. 

For a closer look at the behaviour of QuAcq and MQuAcq with different value ordering heuristics for query generation, we evaluated their performance in terms of the time elapsed and the size of the learned network $C_L$ in relation to the number of queries posted to the user. 
Figures~\ref{fig:beh1} and~\ref{fig:beh2} illustrate the performance of the two algorithm on Sudoku and Latin when using random value ordering, while Figures~\ref{fig:beh3} and~\ref{fig:beh4} illustrate their performance when using {\em max$_v$} for value ordering.

\begin{figure}[h]

\subfloat[]{\includegraphics[width=2.2in]{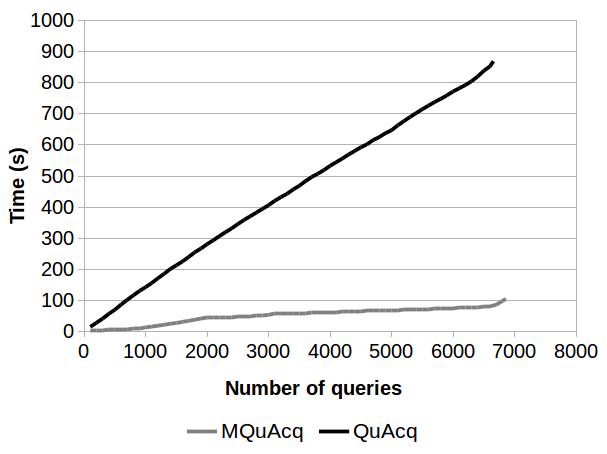}}
\qquad
\subfloat[]{\includegraphics[width=2.2in]{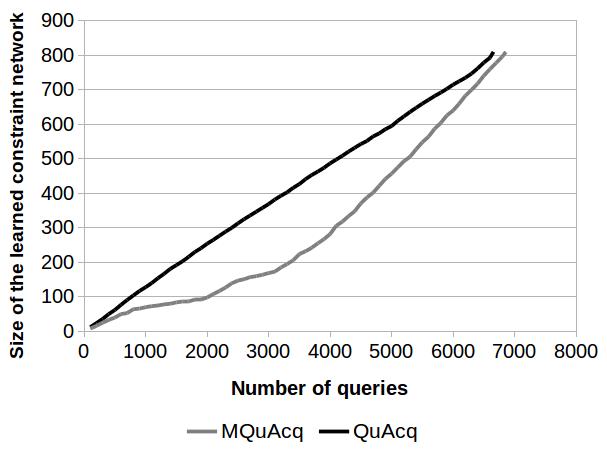}}

\caption{The behaviour of QuAcq and MQuAcq in Sudoku, using the {\em max$_B$} heuristic, {\em bdeg} for variable ordering and random value ordering.}

\label{fig:beh1}
\end{figure}

\begin{figure}[h]

\subfloat[]{\includegraphics[width=2.2in]{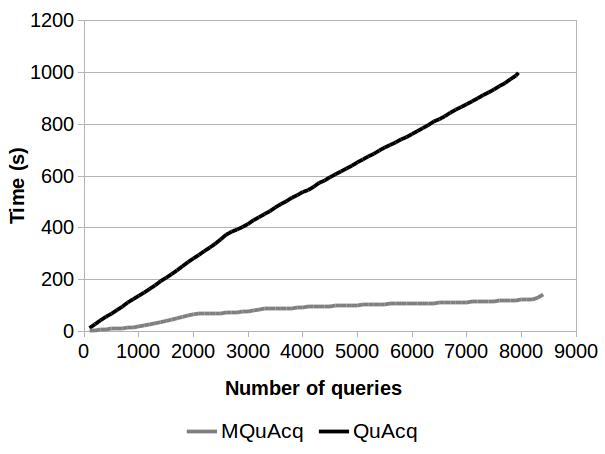}}
\qquad
\subfloat[]{\includegraphics[width=2.2in]{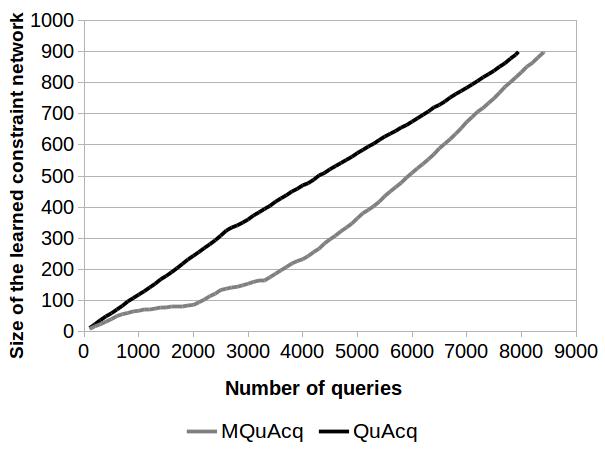}}

\caption{The behaviour of QuAcq and MQuAcq in Latin, using the {\em max$_B$} heuristic, {\em bdeg} for variable ordering and random value ordering.}

\label{fig:beh2}
\end{figure}

\begin{figure}[h]

\subfloat[]{\includegraphics[width=2.2in]{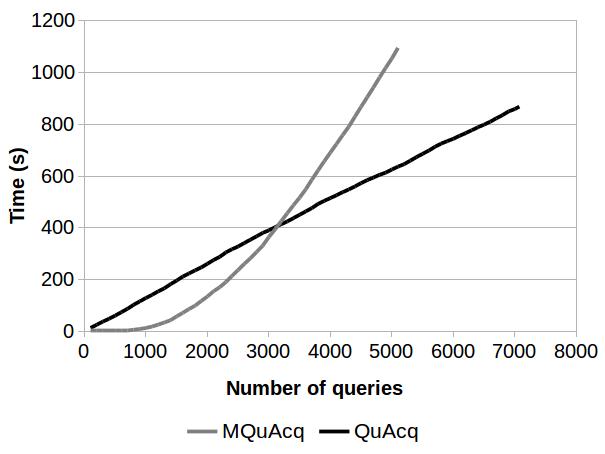}}
\qquad
\subfloat[]{\includegraphics[width=2.2in]{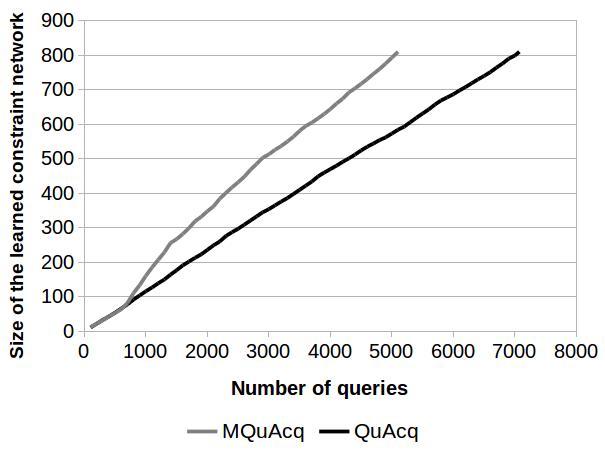}}

\caption{The behaviour of QuAcq and MQuAcq in Sudoku, using the {\em max$_B$} heuristic, {\em bdeg} for variable ordering and {\em max$_v$} for value ordering.}

\label{fig:beh3}
\end{figure}

\begin{figure}[h]

\subfloat[]{\includegraphics[width=2.2in]{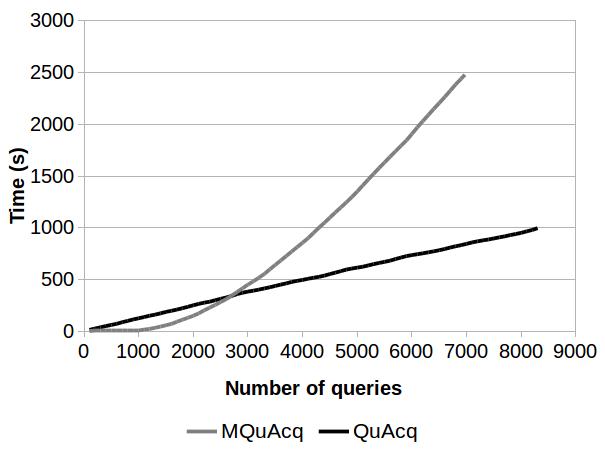}}
\qquad
\subfloat[]{\includegraphics[width=2.2in]{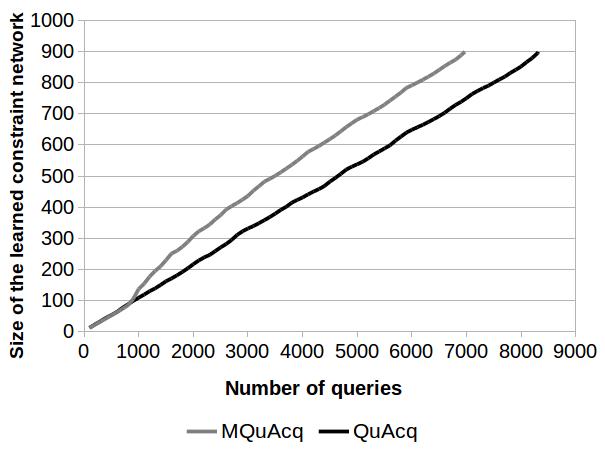}}

\caption{The behaviour of QuAcq and MQuAcq in Latin, using the {\em max$_B$} heuristic, {\em bdeg} for variable ordering and {\em max$_v$} for value ordering.}

\label{fig:beh4}
\end{figure}

In Figures~\ref{fig:beh1} and~\ref{fig:beh2} we can observe that using random value ordering QuAcq needs fewer queries to learn a higher proportion of the target network than MQuAcq. In contrast, MQuAcq needs far less time as it learns all the violated constraints from the target network from each generated query. 
When using {\em max$_v$} for value ordering, these results are reversed (Figures~\ref{fig:beh3} and~\ref{fig:beh4}). That is, MQuAcq takes more time than QuAcq as the process unfolds, but it requires fewer queries to learn the target network. This reversal occurs because with the {\em max$_v$} heuristic MQuAcq can acquire more information (i.e. more constraints) from each generated query. This leads to fewer queries but at the same time it needs more time due to the branching of {\em FindAllCons} as explained before.

A generic remark we can make regarding the value ordering heuristic in MQuAcq is that it can be selected depending on which metric of the constraint acquisition process is viewed as critical. If the only important factors are the number of queries posted to the user and the size of the queries, the {\em max$_v$} heuristic is a better option than random ordering. This is often the case when the user is human. On the other hand, if speeding up the acquisition process is more important, random value ordering should be preferred. This can occur in cases where the user is an existing software system. Concerning QuAcq, we can see that {\em max$_v$} does not improve the acquisition process in any metric. On the contrary, it increases the number of queries, as locating the scope of a violated constraint can end up in a lot small positive queries.

\subsection{The effect of the size of the bias}
\label{sec:bias}

We evaluated the effect of the size of the bias on MQuAcq, in terms of the number of queries posted and the time needed to converge. We used the constraint relations needed for each problem and increased the size of the bias progressively using the language $\{=, \neq, >, <, \leq, \geq, x_i - x_j = 1, |x_i - x_j| = 1, |x_i - x_j| > y, |x_i - x_j| = y, |\lfloor x_i/3 \rfloor - \lfloor x_j/3 \rfloor| > y\}$. We used ExamTT (Figure~\ref{fig:bias_tt}), Latin (Figure~\ref{fig:bias_latin}), Sudoku (Figure~\ref{fig:bias_sudoku}) and Golomb (Figure~\ref{fig:bias_golomb}) to evaluate the effect of the bias' size in different problems.

\begin{figure}[H]

\subfloat[]{\includegraphics[width=2.2in]{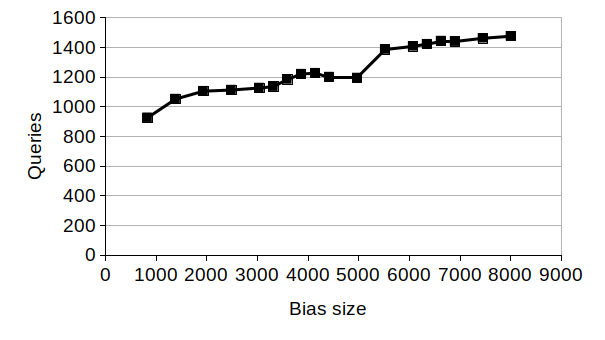}}
\qquad
\subfloat[]{\includegraphics[width=2.2in]{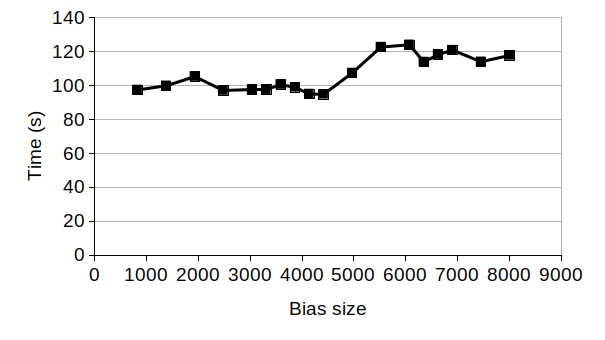}}

\caption{Performance of MQuAcq in Exam Timetabling with bias of different sizes.}

\label{fig:bias_tt}
\end{figure}

\begin{figure}[H]

\subfloat[]{\includegraphics[width=2.2in]{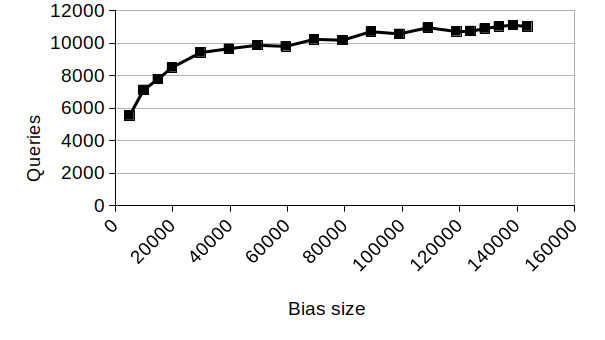}}
\qquad
\subfloat[]{\includegraphics[width=2.2in]{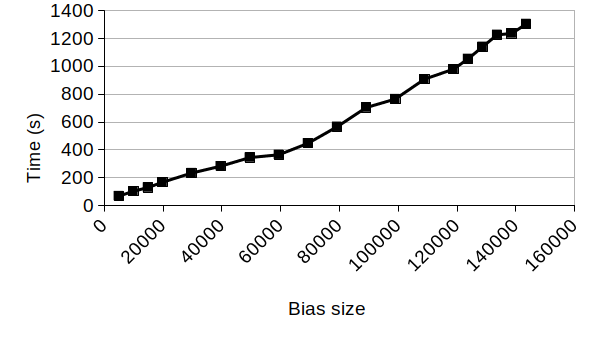}}

\caption{Performance of MQuAcq in Latin squares with bias of different sizes.}

\label{fig:bias_latin}
\end{figure}

\begin{figure}[h]

\subfloat[]{\includegraphics[width=2.2in]{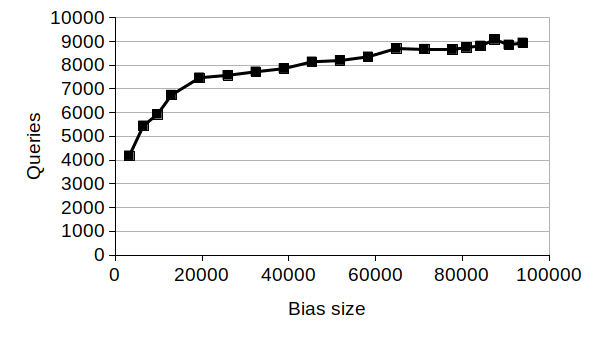}}
\qquad
\subfloat[]{\includegraphics[width=2.2in]{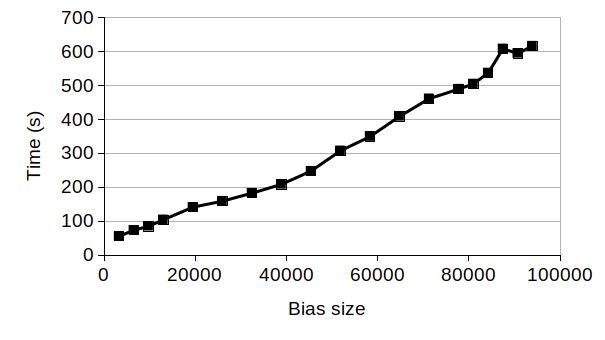}}

\caption{Performance of MQuAcq in Sudoku with bias of different sizes.}

\label{fig:bias_sudoku}
\end{figure}

\begin{figure}[h]

\subfloat[]{\includegraphics[width=2.2in]{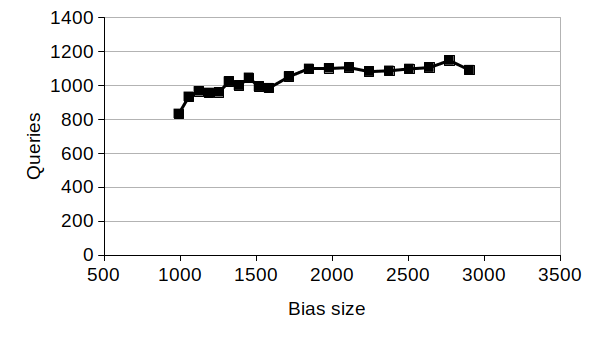}}
\qquad
\subfloat[]{\includegraphics[width=2.2in]{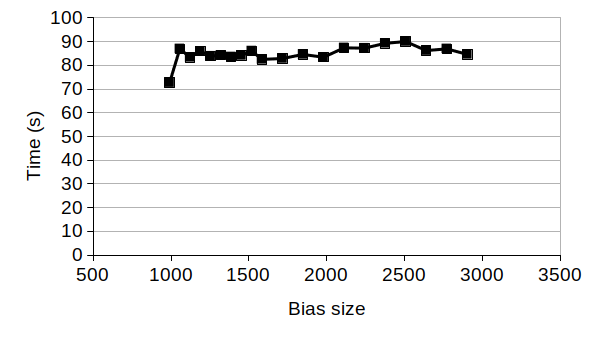}}

\caption{Performance of MQuAcq in Golomb rulers with bias of different sizes.}

\label{fig:bias_golomb}
\end{figure}

Looking at the effect of the bias on the number of queries, we can observe that it does not affect it considerably. In all the benchmarks, the number of queries increases logarithmically as the size of the bias is increased. This is more visible in Latin and Sudoku, where the increase in the bias size is larger, because of the larger number of variables in these benchmarks. These results agree with the corresponding results given in \cite{bessiere2013constraint}. The increase in the number of queries is very mild because although in the worst case each positive query will remove only one constraint from $B$ (in which case the increase in the number of queries would be substantial), in practice, each positive query removes several constraints from $B$, even in the same scope. Thus, as the number of constraints in $B$ increases, so does the average number of constraints removed by each positive query, resulting in a mild increase in the total number of queries.

Regarding the effect of the bias on the time required by MQuAcq, in problems with fewer variables, where the number of constraints in the bias is low even when the whole language is considered, the increase in the time needed to converge is very small (around 25s in Exam Timetabling and 20s in Golomb). However, in Sudoku and Latin, where the larger number of variables means that the size of the bias increases considerably when taking into account more relations in the language used, the increase in the time needed is sharper. But still, this increase is manageable. Overall, the results show that learning problems with expressive biases scales well, even when using a large language to construct the bias, especially regarding the number of generated queries.

\subsection{Scalability Analysis}
\label{sec:scaling}

Finally, we ran experiments to investigate the scalability of MQuAcq as the problem size increases. Towards this we used the following benchmarks, with instances of different sizes:

\textbf{Latin Square}. We used instances with the number of rows/columns $n = 6, ..., 12$. The language used is the same as above, i.e. $\Gamma = \{=, \neq, >, < \}$. Thus, the number of variables varied from 36 to 144 and the size of the target network from 180 to 1584 constraints.

\textbf{Exam Timetabling}. We used instances with the number of variables (number of courses) varying from 24 to 54. The size of the target network varied from 276 to 1431 constraints. The language used is the same as the one described above.

\textbf{Radio Link Frequency Assignment Problem}.
We used simplified versions of the problem, with 40, 45, 50, 55 and 60 variables. The size of the target network varied from 52 to 170 
 binary distance constraints. The language used was the same as above ($\{|x_i - x_j| > y, |x_i - x_j| = y\}$, with 5 different possible values for $y$).

We ran MQuAcq with {\em FindScope-2}, using max$_B$ as the optimization heuristic. We used {\em bdeg} for variable ordering with {\em random} value ordering. We evaluated our algorithm in terms of the number of queries posted to the user and its time performance. The results are shown in Figure~\ref{fig:scale_latin} for Latin, Figure~\ref{fig:scale_rlfap} for RLFAP and Figure~\ref{fig:scale_tt} for the Exam Timetabling problem.

\begin{figure}[h]

\subfloat[]{\includegraphics[width=2.2in]{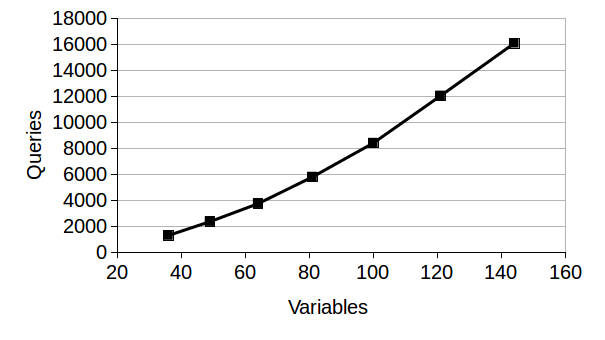}}
\qquad
\subfloat[]{\includegraphics[width=2.2in]{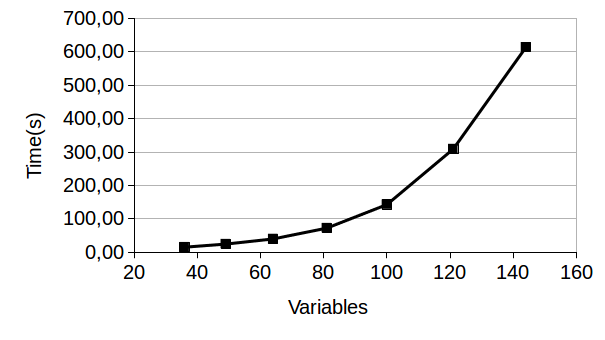}}

\caption{Performance of MQuAcq in Latin instances of different size}

\label{fig:scale_latin}
\end{figure}

\begin{figure}[h]

\subfloat[]{\includegraphics[width=2.2in]{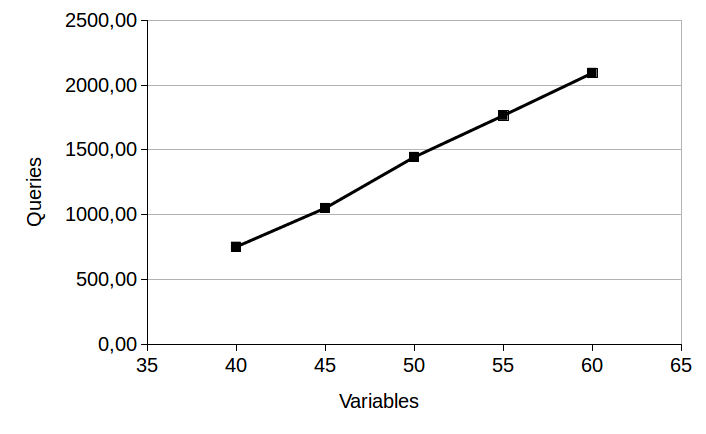}}
\qquad
\subfloat[]{\includegraphics[width=2.2in]{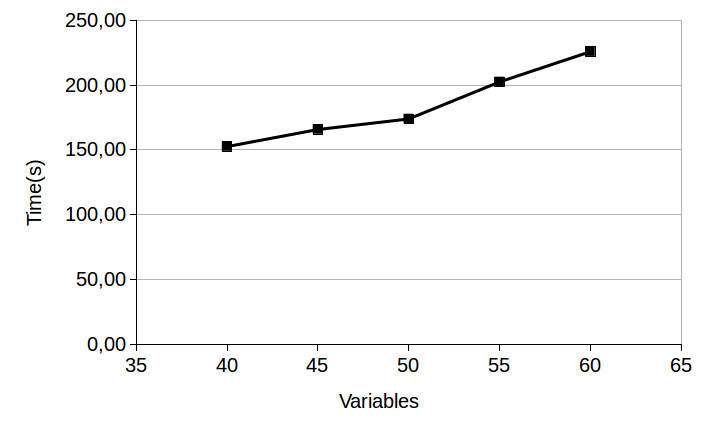}}

\caption{Performance of MQuAcq in RLFAP instances of different size}

\label{fig:scale_rlfap}
\end{figure}

\begin{figure}[h]

\subfloat[]{\includegraphics[width=2.2in]{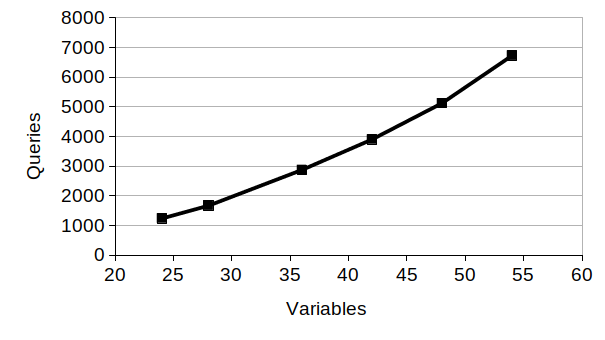}}
\qquad
\subfloat[]{\includegraphics[width=2.2in]{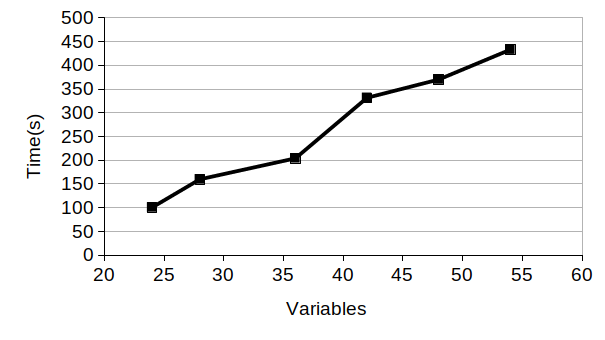}}

\caption{Performance of MQuAcq in Exam Timetabling instances of different size}

\label{fig:scale_tt}
\end{figure}

As we can observe, the increase in the number of queries is proportional to the increase in the number of variables for all benchmarks. This confirms our theoretical analysis. Focusing on the time performance, we can see that the time needed for convergence in Latin problems rises sharply as the number of variables grows beyond 100. This can be explained by the substantial increase in the number of constraints of the target network in these instances. On the other hand, in RLFAP, the increase in time is not very significant because the number of constraints remains relatively small. In the Exam Timetabling problem we see that the time needed is analogous to the number of queries, and grows proportionally to the number of variables present in the problem.

Hence, our proposed algorithm scales up quite well in terms of the number of queries required, while the time performance, being highly dependant on the size of the target network, can rise sharply, and even become unmanageable, for target networks with large numbers of constraints. We believe that methods that try to exploit the structure of the problem being learned may help alleviate this problem, and we intend to work on this in the immediate future.

\section{Discussion}
\label{sec:disc}

We now discuss certain aspects of MQuAcq in relevance to its performance, theoretical guarantees, and applicability. First, we discuss the importance of partial queries, which is a strong point of the algorithm. Then we elaborate on a negative result by proving that MQuAcq cannot learn constraint networks with an optimal number of queries even for very simple languages. Finally, we discuss a weakness of all the proposed constraint acquisition algorithms which paves the way for future work.

\subsection{On partial queries}

Given the importance of partial queries in MQuAcq (and QuAcq), a question that arises is whether such queries are easier or harder for the user to classify than complete ones. 

First of all, a partial example does not have to be part of a complete solution to be classified as positive. The user only needs to decide if the example at hand violates any requirement (constraint). Hence, it can be easier for the user to classify small examples with only a few variables instead of full assignments, simply because inspecting if a full assignment satisfies all the requirements can be very tedious. This is especially true when the problem is large, consisting of many variables. 
Hence, partial queries can make the acquisition process easier for the user.
In addition, the smaller negative examples posted to the user require fewer queries from {\em FindScope-2} to locate the scope of the violated constraint.

Another important factor that supports the argument that partial queries are easier to classify is that many partial queries are subsets of the same complete negative query. Thus, they may be easier for the user to classify simply because the user has already seen the full query and has determined that it is not a solution. 
As a result, generating a new example is not always the best choice if we have not acquired the desired information from the previous generated one. 

Considering the above, the way MQuAcq operates, posting more partial sub-queries than generating new ones, is favorable for the user, not only because of the reduced waiting time but also because it makes it easier to answer the queries.

A drawback of MQuAcq is that searching for all the violated constraints of each generated example can lead to posting a lot of relatively small positive queries that violate only a few constraints from the bias. However, it is desirable to prune the bias from the constraints that are not included in $C_T$ with as few queries as possible. Thus, ideally, we want each positive query to violate a maximum number of constraints from $B$ to prune it with just a few queries. This drawback of MQuAcq is the main reason for the slightly increased number of queries it posts compared to QuAcq when both of them learn the complete target network including the redundant constraints with the {\em max$_B$} heuristic (Table~\ref{res:maxb}). 

This problem could be avoided if the algorithm focused on some of the violated constraints by the generated example, instead of trying to acquire all of them. Non-random problems usually display some structure/pattern in their constraint network. However, this is not taken into account by the existing constraint acquisition algorithms. As future work, as we mentioned above, we plan to adjust the acquisition process to take into account the structure that is revealed as constraints are learned and hence target specific constraints.

\subsection{On the (non-)optimality of MQuAcq}


An interesting question about concept learning algorithms is whether they can learn certain types of concepts with an optimal number of queries. In~\cite{bessiere2013constraint} it was proved that QuAcq is guaranteed to converge after $O(|X| \cdot log(|X|))$ queries in the languages $\{=, \neq\}$ and $\{>\}$ on the Boolean domain. However, this is not the case for MQuAcq.
We now show that MQuAcq is not optimal even for very simple languages like the one that includes a basic binary relation. 

\begin{prop} 
\label{optimality-mquacq}

MQuAcq, using max for the generation of the queries, does not learn Boolean networks on the language $\{=\}$ with an optimal number of queries.

\end{prop}

\begin{proof} 

In~\cite{bessiere2013constraint}, it is proved that the minimum number of queries required to learn a constraint network in this language is in $\Omega(|X| \cdot log|X|)$. In such a network, the maximum number of constraints is equal to the number of $2$-combinations of the $|X|$ variables which is $\frac{|X|\cdot(|X|-1)}{2}$.  For each constraint a number of queries up to $2 \cdot |S| \cdot log|X|$ may be needed by {\em FindScope-2} (Proposition~\ref{prop:findscope-2_compl}). Assume that MQuAcq generates an example that violates all the constraints. {\em FindAllCons} will find all the minimal scopes and learn all the constraints. Thus, in this case it will learn both the redundant and the non-redundant constraints, i.e. it will learn $\frac{|X|\cdot(|X|-1)}{2}$ constraints. As a result, in the worst case MQuAcq's number of queries to learn the constraint network is in $\Omega(|X|^2 \cdot log|X|)$, which is not optimal.  

\end{proof}

The following example illustrates this behaviour of MQuAcq, contrasting it to QuAcq.

\begin{example}
\label{ex:optim}

Consider a problem consisting of $4$ variables with domains $\{ 1, 2 \}$. Also, assume that the target network consists of a single clique of $=$ constraints, i.e. $C_T = \{ =_{12}, =_{13}, =_{14}, =_{23}, =_{24}, =_{34} \}$. Note that there exist equivalent networks to $C_T$ with fewer constraints. For instance, the first three constraints are enough to form an equivalent network, and in this case, the other three are implied (i.e. they are redundant).

As MQuAcq will learn the entire target constraint network, including the redundant constraints, it will need $|C_T| * log(|X|) = 6*2 = 12$ queries to find the scopes of the constraints.  
No query is needed to be made by FindC, as we have $|\Gamma| = 1$. On the other hand, QuAcq will not learn the redundant constraints. 
The number of non-redundant constraints in the above constraint network is equal to $|X| -1 = 3$. Thus, QuAcq will need $3*2 = 6$ queries to learn the constraint network and converge.

\end{example}

We now prove that if $max_B$ is used by either QuAcq or MQuAcq, allowing them to generate partial queries (e.g. at line 5 of MQuAcq or line 4 of QuAcq), then these algorithms are not optimal in terms of the number of queries posted to the user even on the very simple language $\{=\}$, on which QuAcq (with {\em max}) is optimal.

\begin{prop} 
\label{optimality-partial}

Constraint acquisition algorithms QuAcq and MQuAcq do not learn Boolean networks on the language $\{=\}$ with an optimal number of queries, if partial queries can be generated.

\end{prop}

\begin{proof} 

The minimum number of queries required to learn a constraint network in this language is in $\Omega(|X| \cdot log|X|)$~\cite{bessiere2013constraint}. The maximum number of constraints is equal to the number of $2$-combinations of the $|X|$ variables which is $\frac{|X|\cdot(|X|-1)}{2}$.  For each constraint a number of queries up to $2 \cdot |S| \cdot log|X|$ is needed by {\em FindScope-2} (Proposition~\ref{prop:findscope-2_compl}). In case that partial queries can be generated, redundant constraints can be learned too, so in the worst case the number of queries to learn the constraint network is in $\Omega(|X|^2 \cdot log|X|)$, which is not optimal.  

\end{proof}

\subsection{On errors and omissions}

One significant issue that has not been addressed in the context of constraint acquisition is the possibility of omissions and/or errors in the answers of the user to the posted queries. All the constraint acquisition algorithms that have been proposed are guaranteed to operate only under the assumption that the queries are answered correctly. 

In the context of concept learning the existence of omissions or errors has been studied for some classes of concepts. Angluin et. al presented an algorithm that can learn the target concept function by using equivalence and incomplete membership queries ~\cite{angluin1994randomly} . In this model, the answers to some of the learner's membership queries may be unavailable. Extending this, in the exact learning model defined in~\cite{angluin1997malicious} the learning system can learn exactly a target concept using equivalence and membership queries with at most some number $l$ of errors or omissions in the answers of the user to the membership queries posted. 

In this model the {\em limited} membership queries and the {\em malicious} membership queries are introduced. A limited membership query may be answered either by classifying the example (correctly), or with a special answer i.e. ``I don't know'', while in a malicious membership query the classification by the user may be wrong.
Although equivalence queries are more difficult to be answered than membership queries and more expensive~\cite{bshouty1996asking}, in the above model the assumption is that the answers to equivalence queries remain correct, meaning that any counterexample returned is indeed a counterexample to the hypothesis of the learning algorithm.

Extending the above models to more classes, \cite{bisht2008learning} showed that for concepts that are closed under projection both models are equivalent to the exact learning model without omission and errors. In addition, the presented system can also handle errors in the equivalence queries, i.e. the {\em malicious} equivalence query (MEQ) is introduced, in which the user can return a wrong counterexample (an assignment that is not a counterexample in the hypothesis) for at most $l$ different assignments.

One relevant question regarding the answers to the queries is whether the omissions or errors are {\em persistent} or not. They are persistent if the same query to the same examples always returns the same answer (even if the answer is an omission or if it is wrong). In the above models the assumption is that the answers are persistent. 
A model of non-persistent errors is defined by Sakakibara~\cite{sakakibara1991learning}, in which the answer to each query may be wrong with some given probability. In this model, repeated membership queries for the same example are considered as independent events, so the answer may be different. In this model, a general technique of repeating each query sufficiently often to establish the correct answer with high probability is introduced.

Although in the constraint acquisition context we only have membership and partial queries (as explained, equivalence queries are considered too hard to be answered by the user), the presence of omissions or errors has not been studied yet. We plan to deal with this significant issue in the future. Of course in this context limited and malicious partial queries have to be examined too.

\section{Conclusion}
\label{sec:conclusion}

Constraint acquisition has started to receive increasing attention as a useful tool for automated problem modeling in CP. As a result, a number of both passive and active acquisition algorithms have been proposed, with QuAcq and MultiAcq being prime examples of active algorithms. However, two bottlenecks of such algorithms are the large number of queries required to converge to the target network, and the high cpu times needed to generate queries, especially near convergence. An additional side effect of the latter is the often occurrence of premature convergence in constraint acquisition systems.

We have presented new methods that can boost the performance of active constraint acquisition systems. We proposed the MQuAcq algorithm which extends QuAcq to discover all the violated constraints from a negative example, just like MultiAcq does, but with a better complexity bound in terms of the number of queries. We also proposed an optimization on the process of locating scopes that, as experiments demonstrate, helps reduce the number of queries by up to $85\%$ in some cases.

Another contribution of our work is that we focus on query generation which is a very important but rather overlooked part of the acquisition process. We described the algorithmic query generation process of standard interactive acquisition systems in detail, and we proposed several heuristics that can be applied during query generation to boost the performance of constraint acquisition algorithms.

Experimental results demonstrate that an algorithm which integrates all our methods significantly outperforms the state-of-the-art active constraint acquisition algorithms on all the important metrics. It does not only generate considerably fewer queries than QuAcq and MultiAcq, but it is also by far faster than both of them, both in average query generation time and in total run time.
Last but not least, our proposed heuristics for the query generation process support the generation of more ``informative'' queries and also largely alleviate the premature convergence problem.

As future work, it would be very interesting if a hybrid system that integrades a passive learning method, specifically ModelSeeker, and an active one, such as MQuAcq, was designed and built. We believe that the two approaches are orthogonal, and combining their strengths may prove very beneficial in practice. ModelSeeker can learn constraints in highly structured problems using only very few examples, but this is not the case in problems with irregular structure. On the other hand, active methods, such as our own, require to generate a much larger number of examples to learn problems like Sudoku, but can handle irregularly structured problems. So ideally, in the future we would like to have a hybrid system that takes as input a (small) set of examples, runs ModelSeeker to learn the basic constraints, and then completes the model using an active technique.

\bibliographystyle{spbasic}      
\bibliography{paper}   

\end{document}